\documentclass[10pt,twocolumn]{article}
\usepackage[letterpaper,margin=0.75in,columnsep=0.25in]{geometry}
\usepackage{lmodern}
\usepackage{amsmath}
\usepackage{authblk}

\usepackage[prevent-all]{widows-and-orphans}
\usepackage[round]{natbib}

\usepackage[utf8]{inputenc}
\usepackage[T1]{fontenc}
\usepackage{hyperref}
\usepackage{cleveref}
\usepackage{url}
\usepackage{booktabs}
\usepackage{amsmath}
\usepackage{tcolorbox}
\usepackage{amssymb}
\usepackage{mathtools}
\usepackage{amsthm}
\usepackage{nicefrac}
\usepackage{microtype}
\usepackage{xcolor}
\usepackage{algorithm}
\usepackage{algpseudocode}
\usepackage{subcaption}
\usepackage{placeins}
\usepackage{balance}
\usepackage{makecell}
\usepackage{multirow}

\newtheorem{proposition}{Proposition}

\newtheorem{corollary}{Corollary}
\usepackage{etoolbox}
\newcommand{\keepstatementtogether}{\interlinepenalty=10000 \predisplaypenalty=10000 \postdisplaypenalty=10000}
\AtBeginEnvironment{proposition}{\keepstatementtogether}
\AtBeginEnvironment{corollary}{\keepstatementtogether}
\AtBeginEnvironment{theorem}{\keepstatementtogether}
\AtBeginEnvironment{lemma}{\keepstatementtogether}
\AtBeginEnvironment{definition}{\keepstatementtogether}

\newcommand{\propref}[1]{\textcolor{gray}{\scriptsize Prop.~\ref{#1}}}

\newcommand{\agg}{\mathop{\Phi}\limits}
\newcommand{\aggX}{\mathop{\Phi_\mathcal{X}}\limits}
\newcommand{\aggY}{\mathop{\Phi_\mathcal{Y}}\limits}
\newcommand{\aggS}{\mathop{\Phi_\mathcal{S}}\limits}

\newcommand{\KL}{\mathrm{d}_{\mathrm{KL}}}

\newcommand{\Pemp}{P_n}

\newcommand{\Pref}{P_{\mathrm{ref}}}
\newcommand{\U}{\mathcal{U}}

\newcommand{\Hspace}{\mathcal{H}}

\newcommand{\R}{\mathbb{R}}

\DeclareMathOperator*{\argmin}{arg\,min}

\newtheorem{remark}{Remark}

\usepackage{tikz}
\usetikzlibrary{arrows.meta, positioning, fit, backgrounds, calc}

\usepackage{amsmath,amssymb}

\usepackage{colortbl}
\usepackage{array}

\definecolor{clGreen}{HTML}{d5e8d4}
\definecolor{clGreenBorder}{HTML}{82b366}
\definecolor{clGreenLight}{HTML}{f9fff9}
\definecolor{clOrange}{HTML}{ffe6cc}
\definecolor{clOrangeBorder}{HTML}{d6a021}
\definecolor{clOrangeLight}{HTML}{fff9f0}
\definecolor{clPurple}{HTML}{e1d5e7}
\definecolor{clPurpleBorder}{HTML}{9673a6}
\definecolor{clPurpleLight}{HTML}{f8f0ff}
\definecolor{clBlue}{HTML}{dae8fc}
\definecolor{clBlueBorder}{HTML}{6c8ebf}
\definecolor{clBlueLight}{HTML}{f0f7ff}
\definecolor{clGray}{HTML}{f5f5f5}
\definecolor{clGrayBorder}{HTML}{666666}
\definecolor{clNovel}{HTML}{FF8000}
\definecolor{clNovelLight}{HTML}{fff8f0}

\tikzset{
  every node/.style={font=\small},
  box/.style={
    draw, rounded corners=3pt,
    text width=#1,
    align=center,
    inner sep=5pt,
    minimum height=1.6em
  },
  header/.style={box=#1, thick, font=\small\bfseries,minimum height=1.5cm},
  sub/.style={box=#1, dashed, font=\small},
  iobox/.style={box=2.0cm, draw=clGrayBorder, fill=clGray, thick, align=center,minimum height=1.5cm},
  arr/.style={-Stealth, thick, gray!70!black},
}

\hypersetup{colorlinks=true, linkcolor=blue!50!black, citecolor=blue!50!black, urlcolor=blue!50!black}

\mathchardef\UrlBreakPenalty=100 \mathchardef\UrlBigBreakPenalty=100
\begin{document}

\title{Unification and Optimization of Robust Supervised Learning}
\author[1,2,3]{Jonas Hanselle\thanks{Equal contribution.}}
\author[1,2]{Valentin Margraf\protect\footnotemark[1]}
\author[1,2]{Clemens Damke}
\author[1,2,3]{Eyke H\"ullermeier}
\affil[1]{LMU Munich}
\affil[2]{Munich Center for Machine Learning (MCML)}
\affil[3]{German Research Center for Artificial Intelligence (DFKI)}
\date{}
\maketitle

\begin{abstract}
Robust learning methods such as distributionally robust optimization, vicinal risk minimization, Mixup, and label smoothing address different sources of uncertainty, including distribution shift, finite-sample effects, and label uncertainty. In practice, however, several such sources may occur simultaneously, and it remains unclear which robust method or combination of methods to use. We introduce a unified framework that organizes a broad class of robustness methods along explicit design axes and makes their relations and combinations explicit within a shared formulation. Within this framework, we analyze how robustness methods interact: some are redundant under particular modeling assumptions, whereas others yield genuinely distinct solution families, and these relationships can change with the model class or loss. We further derive a tractable training procedure that exposes the design choices and their strengths as tunable hyperparameters. Across vision and reward-modeling tasks, the preferred robustness method depends on the task, target metric, and validation data, while suitably selected combinations often outperform the best single method.
\end{abstract}

\section{Introduction}
\label{sec:introduction}
\begin{figure}
    \centering
    \includegraphics[width=0.75\linewidth]{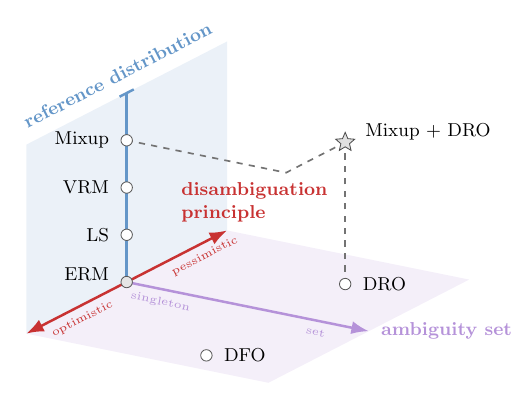}
    \caption{\textbf{Our framework} organizes existing robust learning methods along three axes: reference distribution, ambiguity set, and disambiguation principle, making their relationships explicit and enabling new compositions, e.g. Mixup + DRO.}
    \label{fig:placeholder}
\end{figure}
A supervised learning process typically involves a series of decisions whose consequences are difficult to predict. First, training data is collected and labels are assigned. Then, a model class is chosen, the learning task is formalized, and a predictor is trained on the data. Finally, the predictor is deployed in an environment that the practitioner can rarely fully anticipate. Each of these steps presents potential pitfalls and is prone to failures and suboptimal decisions.

For instance, consider reward modeling from human preferences, a central component in aligning modern language models \citep{ziegler-arxiv19a,kaufmannsurvey}.
A practitioner building such a reward model faces multiple, potentially interacting sources of uncertainty simultaneously. First, training data is finite and often not fully representative: annotators cover only a small portion of the prompt space, leaving large regions unobserved~\citep{stiennon-neurips2020a}. Second, labels are inherently noisy and subjective: different annotators apply different criteria, and even the same annotator may be inconsistent over time~\citep{JEONG2023102992}. Third, deployment conditions differ from the annotation setting: models are evaluated on prompts that vary in topic, style, language, and implicit standards.

A wide range of methods
have been proposed under the umbrella term of ``robust'' supervised learning, using different techniques to maintain predictive performance when faced with such difficulties. Finite-sample degeneracy is mitigated by enriching the training distribution in input space via data augmentation and vicinal techniques~\citep{chapelleVicinalRiskMinimization2000,krizhevsky-nips12a,zhangMixupEmpiricalRisk2018,yun-iccv19a,hendrycks2020augmix}. Label uncertainty is addressed through softening or relaxing labels, e.g.\ label smoothing~\citep{szegedyRethinkingInceptionArchitecture2016} and label relaxation~\citep{DBLP:conf/aaai/LienenH21}. Distribution shift is handled through adversarial training~\citep{goodfellowadv,madry2018towards} and distributionally robust optimization (DRO)~\citep{sinha2018certifying,blanchetDistributionallyRobustOptimization2025}, as well as its optimistic counterpart, distributionally favorable optimization (DFO)~\citep{jiangDistributionallyFavorableOptimization2024a}.

While these techniques address individual sources of uncertainty, they have largely been developed in isolation and are based on different formulations of the robust learning problem. When several sources of uncertainty are present simultaneously, this makes it difficult to determine which robustness technique to use, whether multiple techniques should be combined, and when such combinations are complementary or redundant.

We introduce a unified framework that represents robust learning methods in a design space spanned by three dimensions, each corresponding to a robustification technique (cf.\ Figure \ref{fig:placeholder}):
(i) the choice of a \emph{reference distribution},
(ii) the choice of an \emph{ambiguity set}, and
(iii) the choice of a \emph{disambiguation principle}.
Moreover, we propose a tractable training procedure that exposes these choices as hyperparameters of a single objective function. This enables joint hyperparameter optimization (HPO) to find the best configuration of robustness techniques for a given task.

Our \textbf{contributions} can be summarized as follows:

\begin{itemize}

    \item \textbf{Unified framework.}
    We formulate robust supervised learning along three explicit design axes: (i) the reference distribution, (ii) the ambiguity set, and
    (iii) the disambiguation principle.
    This places a broad class of existing robustness techniques and their
    compositions in a common design space.
    \item \textbf{Theoretical analysis.}
    In linear least squares, we characterize how several robustness techniques
    interact, showing that some compositions are redundant up to
    reparameterization whereas others yield genuinely distinct solution
    families.
    We further show that these relationships need not persist under different
    losses or hypothesis classes.
    \item \textbf{Unified training and empirical study.}
We derive a modular training procedure that exposes robustness choices as
tunable hyperparameters, enabling data-driven selection of individual
methods and their combinations.
Across vision and reward-modeling tasks, combining techniques is often
beneficial, while the preferred combination varies with the specific setting.
\end{itemize}

\section{Robust Learning}
\label{sec:robust_learning}

In supervised learning, we consider an input space $\mathcal{X}$,
an output space $\mathcal{Y}$, and assume that data is sampled from
$P^\star \in \mathcal{P}(\mathcal{X} \times \mathcal{Y})$,
where $\mathcal{P}(\mathcal{X} \times \mathcal{Y})$ denotes the set of probability measures on $\mathcal{X} \times \mathcal{Y}$. Given a hypothesis class
$\Hspace$ and a loss function $\ell : \mathcal{Y} \times
\mathcal{Y} \to \mathbb{R}$, the goal is to find a hypothesis
$h \in \Hspace$ minimizing the population risk,
\begin{align*}
h^\star & = \argmin_{h \in \Hspace} R(h; P^\star) \\
& = \argmin_{h \in \Hspace} \mathbb{E}_{(x,y)\sim P^\star}
\big[ \ell(h(x), y) \big] \, .
\end{align*}
Since $P^\star$ is unknown, one relies on a dataset $\mathcal{D} = \{(x_i, y_i)\}_{i=1}^n$ drawn i.i.d.\ from
$P^\star$ and constructs the empirical distribution
$$
\Pemp := \frac{1}{n}\sum_{i=1}^n \delta_{(x_i,y_i)} \, ,
$$
where $\delta_{(x,y)}$ denotes the Dirac measure at $(x,y)$. Empirical risk
minimization (ERM) then selects a hypothesis
\begin{align*}
\hat{h}_{\mathrm{ERM}}
& = \argmin_{h \in \Hspace} R(h; \Pemp) \\
& = \argmin_{h \in \Hspace} \frac{1}{n}\sum_{i=1}^n
\ell(h(x_i), y_i) \, .
\end{align*}
ERM can be suboptimal in practice for several
reasons~\citep{blanchetDistributionallyRobustOptimization2025}.
First, $\Pemp$ is a poor
proxy for $P^\star$ when data are scarce: it is a degenerate discrete measure that assigns zero mass to unobserved regions of the input space. Second, observed samples may deviate from $P^\star$ due to noise, annotation errors, or contamination. Third, model performance is ultimately evaluated under a deployment distribution $P_{\mathrm{deploy}}$ that may differ from $P^\star$ due to, e.g. distributional shift or adversarial attacks. Together, these issues motivate
learning paradigms that go beyond ERM. In the following, we recall
existing techniques relevant to this work, organized by the
failure mode they address.

\paragraph{Finite-sample degeneracy.}

To mitigate the degeneracy of $\Pemp$, techniques such as vicinal risk minimization (VRM)~\citep{chapelleVicinalRiskMinimization2000}
and Mixup~\citep{zhangMixupEmpiricalRisk2018} smoothen $\Pemp$ and/or enrich the training data with synthetic samples, while broader data-augmentation methods
such as CutMix~\citep{yun-iccv19a} and AugMix~\citep{hendrycks2020augmix}
use modality-specific transformations.

\begin{figure*}[t]
  \centering
  \resizebox{\linewidth}{!}{\begin{tikzpicture}[scale=1.0, every node/.style={font=\small}]

\definecolor{empColor}{RGB}{100,100,100}
\definecolor{refColor}{RGB}{100,150,200}
\definecolor{ambColor}{RGB}{150,100,200}
\definecolor{worstColor}{RGB}{200,50,50}
\definecolor{bestColor}{RGB}{50,150,50}

\begin{scope}[xshift=0cm]
  \node[font=\bfseries] at (2, 2.5) {Empirical};
  \node[font=\small] at (2, 2.05) {$\Pemp$};

  \draw[->] (0, 0) -- (4, 0) node[right, font=\footnotesize] {$x$};
  \draw[->] (0, 0) -- (0, 1.8) node[above, font=\footnotesize] {$p$};

  \foreach \x in {0.75, 1.05, 1.15} {
    \draw[empColor, very thick] (\x, 0) -- (\x, 1.2);
    \fill[empColor] (\x, 1.2) circle (2.5pt);
  }

  \foreach \x in {2.55, 2.85, 3.1} {
    \draw[empColor, very thick] (\x, 0) -- (\x, 1.2);
    \fill[empColor] (\x, 1.2) circle (2.5pt);
  }

\end{scope}

\draw[-Stealth, thick, gray!70] (4.3, 0.9) -- (5.2, 0.9);
\node[above, font=\footnotesize, align=center] at (4.75, 1.05) {augment/};
\node[below, font=\footnotesize, align=center] at (4.75, 0.75) {smooth};

\begin{scope}[xshift=5.5cm]
  \node[font=\bfseries] at (2, 2.5) {Reference};
  \node[font=\small] at (2, 2.05) {$P_{\mathrm{ref}}$};

  \draw[->] (0, 0) -- (4, 0) node[right, font=\footnotesize] {$x$};
  \draw[->] (0, 0) -- (0, 1.8) node[above, font=\footnotesize] {$p$};

  \draw[refColor, very thick, fill=refColor!30]
    plot[smooth, tension=0.7] coordinates {
      (0.3, 0) (0.6, 0.18) (0.75, 0.54) (1.0, 1.08) (1.2, 0.96)
      (1.5, 0.48) (1.8, 0.36) (2.1, 0.42) (2.4, 0.66) (2.6, 0.9)
      (2.85, 1.2) (3.1, 1.08) (3.3, 0.6) (3.5, 0.18) (3.8, 0)
    } -- (0.3, 0);

  \foreach \x in {0.75, 1.05, 1.15, 2.55, 2.85, 3.1} {
    \fill[empColor] (\x, -0.15) circle (1.5pt);
  }
\end{scope}

\draw[-Stealth, thick, gray!70] (9.8, 0.9) -- (10.7, 0.9);
\node[above, font=\footnotesize, align=center] at (10.25, 0.9) {perturb};

\begin{scope}[xshift=11cm]
  \node[font=\bfseries] at (2, 2.5) {Ambiguity Set};
  \node[font=\small] at (2, 2.05) {$\U(P_{\mathrm{ref}})$};

  \draw[->] (0, 0) -- (4, 0) node[right, font=\footnotesize] {$x$};
  \draw[->] (0, 0) -- (0, 1.8) node[above, font=\footnotesize] {$p$};

  \draw[refColor!50, thick, opacity=0.7]
    plot[smooth, tension=0.7] coordinates {
      (0.3, 0) (0.6, 0.18) (0.75, 0.54) (1.0, 1.08) (1.2, 0.96)
      (1.5, 0.48) (1.8, 0.36) (2.1, 0.42) (2.4, 0.66) (2.6, 0.9)
      (2.85, 1.2) (3.1, 1.08) (3.3, 0.6) (3.5, 0.18) (3.8, 0)
    };

  \draw[ambColor!70, very thick, opacity=0.7]
    plot[smooth, tension=0.7] coordinates {
      (0.5, 0) (0.8, 0.18) (0.95, 0.54) (1.2, 1.08) (1.4, 0.96)
      (1.7, 0.48) (2.0, 0.36) (2.3, 0.42) (2.6, 0.66) (2.8, 0.9)
      (3.05, 1.2) (3.3, 1.08) (3.5, 0.6) (3.7, 0.18)
    };

  \draw[ambColor!70, very thick, opacity=0.7]
    plot[smooth, tension=0.7] coordinates {
      (0.1, 0) (0.4, 0.18) (0.55, 0.54) (0.8, 1.08) (1.0, 0.96)
      (1.3, 0.48) (1.6, 0.36) (1.9, 0.42) (2.2, 0.66) (2.4, 0.9)
      (2.65, 1.2) (2.9, 1.08) (3.1, 0.6) (3.3, 0.18) (3.6, 0)
    };

  \draw[ambColor!70, very thick, opacity=0.7]
    plot[smooth, tension=0.7] coordinates {
      (0.3, 0) (0.6, 0.24) (0.75, 0.66) (1.0, 1.32) (1.2, 1.2)
      (1.5, 0.6) (1.8, 0.48) (2.1, 0.54) (2.4, 0.84) (2.6, 1.14)
      (2.85, 1.5) (3.1, 1.32) (3.3, 0.78) (3.5, 0.24) (3.8, 0)
    };

  \draw[ambColor!70, very thick, opacity=0.7]
    plot[smooth, tension=0.7] coordinates {
      (0.3, 0) (0.6, 0.12) (0.75, 0.36) (1.0, 0.72) (1.2, 0.66)
      (1.5, 0.3) (1.8, 0.24) (2.1, 0.3) (2.4, 0.42) (2.6, 0.6)
      (2.85, 0.78) (3.1, 0.72) (3.3, 0.42) (3.5, 0.12) (3.8, 0)
    };

  \draw[ambColor!70, very thick, opacity=0.7]
    plot[smooth, tension=0.7] coordinates {
      (0.8, 0) (1.0, 0.3) (1.15, 0.84) (1.3, 1.2) (1.45, 1.08)
      (1.7, 0.6) (1.95, 0.48) (2.2, 0.54) (2.45, 0.78) (2.65, 1.02)
      (2.9, 1.32) (3.1, 1.14) (3.3, 0.54) (3.5, 0)
    };

  \draw[ambColor!70, very thick, opacity=0.7]
    plot[smooth, tension=0.7] coordinates {
      (0.1, 0) (0.35, 0.12) (0.5, 0.36) (0.75, 0.78) (1.0, 0.72)
      (1.35, 0.36) (1.65, 0.3) (1.95, 0.3) (2.25, 0.48) (2.5, 0.66)
      (2.8, 0.9) (3.1, 0.84) (3.4, 0.48) (3.7, 0.12)
    };
\end{scope}

\draw[-Stealth, thick, gray!70] (15.3, 0.9) -- (16.2, 0.9);
\node[above, font=\footnotesize, align=center] at (15.75, 0.9) {select};

\begin{scope}[xshift=16.5cm]
  \node[font=\bfseries] at (2, 2.5) {Disambiguation};
  \node[font=\small] at (2, 2.05) {$\agg_{Q \in \U(P_{\mathrm{ref}})}$};

  \draw[->] (0, 0) -- (4, 0) node[right, font=\footnotesize] {$x$};
  \draw[->] (0, 0) -- (0, 1.8) node[above, font=\footnotesize] {$p$};

  \draw[refColor!30, thick, opacity=0.5]
    plot[smooth, tension=0.7] coordinates {
      (0.3, 0) (0.6, 0.18) (0.75, 0.54) (1.0, 1.08) (1.2, 0.96)
      (1.5, 0.48) (1.8, 0.36) (2.1, 0.42) (2.4, 0.66) (2.6, 0.9)
      (2.85, 1.2) (3.1, 1.08) (3.3, 0.6) (3.5, 0.18) (3.8, 0)
    };

  \draw[ambColor!50, thick, opacity=0.5]
    plot[smooth, tension=0.7] coordinates {
      (0.1, 0) (0.4, 0.18) (0.55, 0.54) (0.8, 1.08) (1.0, 0.96)
      (1.3, 0.48) (1.6, 0.36) (1.9, 0.42) (2.2, 0.66) (2.4, 0.9)
      (2.65, 1.2) (2.9, 1.08) (3.1, 0.6) (3.3, 0.18) (3.6, 0)
    };

  \draw[ambColor!50, thick, opacity=0.5]
    plot[smooth, tension=0.7] coordinates {
      (0.3, 0) (0.6, 0.24) (0.75, 0.66) (1.0, 1.32) (1.2, 1.2)
      (1.5, 0.6) (1.8, 0.48) (2.1, 0.54) (2.4, 0.84) (2.6, 1.14)
      (2.85, 1.5) (3.1, 1.32) (3.3, 0.78) (3.5, 0.24) (3.8, 0)
    };

  \draw[ambColor!50, thick, opacity=0.5]
    plot[smooth, tension=0.7] coordinates {
      (0.3, 0) (0.6, 0.12) (0.75, 0.36) (1.0, 0.72) (1.2, 0.66)
      (1.5, 0.3) (1.8, 0.24) (2.1, 0.3) (2.4, 0.42) (2.6, 0.6)
      (2.85, 0.78) (3.1, 0.72) (3.3, 0.42) (3.5, 0.12) (3.8, 0)
    };

  \draw[ambColor!50, thick, opacity=0.5]
    plot[smooth, tension=0.7] coordinates {
      (0.8, 0) (1.0, 0.3) (1.15, 0.84) (1.3, 1.2) (1.45, 1.08)
      (1.7, 0.6) (1.95, 0.48) (2.2, 0.54) (2.45, 0.78) (2.65, 1.02)
      (2.9, 1.32) (3.1, 1.14) (3.3, 0.54) (3.5, 0)
    };

  \draw[ambColor!50, thick, opacity=0.5]
    plot[smooth, tension=0.7] coordinates {
      (0.1, 0) (0.35, 0.12) (0.5, 0.36) (0.75, 0.78) (1.0, 0.72)
      (1.35, 0.36) (1.65, 0.3) (1.95, 0.3) (2.25, 0.48) (2.5, 0.66)
      (2.8, 0.9) (3.1, 0.84) (3.4, 0.48) (3.7, 0.12)
    };

  \draw[worstColor, very thick, fill=worstColor!30]
    plot[smooth, tension=0.7] coordinates {
      (0.5, 0) (0.8, 0.18) (0.95, 0.54) (1.2, 1.08) (1.4, 0.96)
      (1.7, 0.48) (2.0, 0.36) (2.3, 0.42) (2.6, 0.66) (2.8, 0.9)
      (3.05, 1.2) (3.3, 1.08) (3.5, 0.6) (3.7, 0.18)
    } -- (3.7, 0) -- (0.5, 0);
\end{scope}

\end{tikzpicture}}\caption{Our three-axis framework (i) enriches the empirical distribution $\Pemp$ into a \emph{reference distribution} $P_{\mathrm{ref}}$, (ii) defines an \emph{ambiguity set}
$\U(P_{\mathrm{ref}})$ of plausible perturbations of $P_{\mathrm{ref}}$,
and (iii) resolves the resulting uncertainty via a
\emph{disambiguation principle} $\agg$.}
\label{fig:framework_illustration}
\end{figure*}

\paragraph{Label uncertainty.}
When labels are unreliable, robustness can be achieved by modifying the label distribution.
Label smoothing~\citep{szegedyRethinkingInceptionArchitecture2016} redistributes probability mass away from the annotated class, mitigating over-confidence and over-fitting~\citep{mueller-neurips2019a}. Label relaxation~\citep{DBLP:conf/aaai/LienenH21}
is a generalization that replaces class labels by sets of admissible distributions, thereby improving calibration.
Although it acts on the data, label relaxation can also be considered as a means to modify the loss function $\ell$, e.g., to turn standard losses such as $L_1$ or $L_2$ into the robust versions such as $\epsilon$-insensitive of Huber loss \citep{mpub313}.

\textbf{Distribution shift.}
When the deployment distribution differs from the training distribution, robustness requires accounting for plausible shifts. At the instance level, adversarial training~\citep{goodfellowadv,madry2018towards,wang-MART,zhang_TRADES,robey-DALE} enforces robustness to local perturbations. At the group-level, Group DRO~\citep{Sagawa2020Distributionally} optimizes performance against the worst-performing predefined subpopulation. At the distribution level, DRO~\citep{blanchetDistributionallyRobustOptimization2025,sinha2018certifying} guards against worst-case shifts that may arise \emph{after} deployment, whereas DFO~\citep{jiangDistributionallyFavorableOptimization2024a} adopts an optimistic view of potentially correctable shifts that occurred before the decision~\citep{blanchetDistributionallyRobustOptimization2025}. Intermediate approaches, such as probabilistically robust learning~\citep{proba_robust_robey22}, interpolate between these regimes.

\textbf{Related work.}
\label{sec:related_work}
Most prior work operates along a single axis of our framework, leaving the others (implicitly) fixed.
Along the ambiguity-set axis, \citet{weihua-icml18a} expose the structure of the ambiguity set as a configurable choice, showing that $f$-divergence DRO with an unstructured set collapses to ERM under classification losses, and \citet{slowik-aistats22a} establish that calibrated DRO is equivalent to ERM on a re-weighted mixture. Along the disambiguation-principle axis, probabilistically robust learning~\citep{proba_robust_robey22} replaces the worst-case supremum with a $\rho$-quantile that interpolates between ERM and adversarial training, tilted ERM~\citep{term2021} reweights samples through an exponential tilt equivalent to the KL-DRO/DFO dual in stage~4 of our pipeline (cf.\ Figure \ref{fig:robust_pipeline}), and \citet{singhDomainGeneralisationImprecise2024a} train an augmented hypothesis $h(x,\lambda)$ that exposes a continuum of CVaR-aggregated solutions.

Work that varies more than one axis remains scarce. DRO-Augment~\citep{DRO-Augment-iclrreject} composes W-DRO with data augmentation, jointly perturbing the reference distribution and the ambiguity set, and \citet{freieslebenGeneralizationTheoryRobustness2023} offer a conceptual analysis of robustness as a multi-place relation without an algorithmic instantiation.
When the dominant source of uncertainty is not known in advance, and benefits depend on how performance is evaluated (average accuracy, worst-group performance, calibration), it is therefore unclear which technique to use, whether several techniques should be combined, and how to tune them.

\section{A Unified Framework}
\label{sec:framework}

\begin{table*}[t]
\centering
\caption{Existing techniques positioned within the unified framework~\eqref{eq:robust_framework}. Each technique corresponds to a particular choice of reference distribution~$P_{\mathrm{ref}}$, ambiguity set~$\mathcal{U}(P_{\mathrm{ref}})$, and disambiguation principle $\Phi$. Detailed propositions and proofs are provided in Appendix~\ref{apx:proofs_recovery}.}
\label{tab:special-cases}
\vskip 0.1in
\centering{\small
\begin{tabular}{>{\raggedright\arraybackslash}p{3.9cm}>{\centering\arraybackslash}p{2.7cm}>{\centering\arraybackslash}p{3.6cm}>{\centering\arraybackslash}p{2.1cm}c}
\toprule
\textbf{Technique} & \textbf{Reference} $P_{\mathrm{ref}}$ & \textbf{Ambiguity set} $\U(P_{\mathrm{ref}})$ & \textbf{Disambig. principle} $\Phi$ & \\
\midrule
Label smoothing \citep{szegedyRethinkingInceptionArchitecture2016}
  & Smoothed $\tilde{P}_{\mathrm{ls}}$
  & Singleton $\{\tilde{P}_{\mathrm{ls}}\}$
  & -
  & \propref{prop:label-smoothing} \\[2pt]
Label relaxation \citep{DBLP:conf/aaai/LienenH21}
  & Empirical $\Pemp$
  & TV ball over labels, radius $\alpha$
  & $\inf$
  & \propref{prop:label-relaxation} \\[2pt]
DRO \citep{blanchetDistributionallyRobustOptimization2025}
  & Empirical $\Pemp$
  & Statistical ball (e.g.\ Wasserstein), radius $\varrho$
  & $\sup$
  & \propref{prop:dro} \\[2pt]
DFO \citep{jiangDistributionallyFavorableOptimization2024a}
  & Empirical $\Pemp$
  & Statistical ball, radius $\varrho$
  & $\inf$
  & \propref{prop:dfo} \\[2pt]
Mixup \citep{zhangMixupEmpiricalRisk2018}
  & Mixup $\tilde{P}_{\mathrm{mix}}$
  & Singleton $\{\tilde{P}_{\mathrm{mix}}\}$
  & -
  & \propref{prop:mixup} \\[2pt]
VRM \citep{chapelleVicinalRiskMinimization2000}
  & Vicinal $\tilde{P}_{\mathrm{vic}}$
  & Singleton $\{\tilde{P}_{\mathrm{vic}}\}$
  & -
  & \propref{prop:vrm} \\
Group DRO \citep{Sagawa2020Distributionally}
  & Empirical $\Pemp$ (group prop.)
  & KL ball over group weights, radius $\varrho$
  & $\sup$
  & \propref{prop:group-dro} \\[2pt]
Group re-weighting
  & Re-weighted $\tilde{P}_{\mathrm{grp}}$
  & Singleton $\{\tilde{P}_{\mathrm{grp}}\}$
  & -
  & \propref{prop:group-dro} \\[2pt]
\bottomrule
\end{tabular}}
\end{table*}

As already said, we characterize robust supervised learning along three design axes (cf. Figure~\ref{fig:framework_illustration}):
\begin{enumerate}
    \item \textbf{Reference distribution $P_{\mathrm{ref}}$:}
    The distribution around which perturbations are defined.
    This can be the empirical distribution $\Pemp$ or a smoothed
    or augmented variant thereof.

\item \textbf{Ambiguity set $\U(P_{\mathrm{ref}})$:}
A set of distributions
\[
\U(\Pref)
= \bigl\{ Q \in \mathcal{P}(\mathcal{X}\times\mathcal{Y}) :
\Delta(Q,\Pref) \le \varrho \bigr\}
\]
around $P_{\mathrm{ref}}$,
where $\Delta$ is a distribution distance and $\varrho$ the perturbation radius.
This set specifies \emph{plausible deviations} from the reference $P_{\mathrm{ref}}$.
Typical choices include Wasserstein distance for geometric perturbations
\citep{kuhnWassersteinDistributionallyRobust2019} and KL divergence for
re-weighting the support of $\Pref$
\citep{huKullbackLeiblerDivergenceConstrained2013a}.

\item \textbf{Disambiguation principle $\agg$:}
The principle used to resolve the uncertainty resulting  from $\U(\Pref)$ into a single training objective.
An optimistic principle seeks the best case, appropriate for potentially
correctable contamination, whereas a pessimistic principle guards against
worst-case deployment shifts.
\end{enumerate}

Concrete methods can be obtained in a systematic way by instantiating each of these dimensions. Vice versa, the decomposition isolates the core assumptions underlying a method and enables systematic comparison and combination.
For instance, Wasserstein-DRO uses the empirical distribution $\Pemp$ as a reference, defines plausible deviations as a Wasserstein ball of radius $\varrho$ around that distribution, and resolves the ambiguity by optimizing against the worst case in that ball \citep{kuhnWassersteinDistributionallyRobust2019}.

More generally, given an instantiation in the form of a reference $P_{\mathrm{ref}}$ that anchors the ambiguity set $\U(P_{\mathrm{ref}})$, which in turn defines the distributions over which the disambiguation principle $\agg$ is taken, we define the generic robust objective as follows:
\begin{equation}
    \hat{h}_{\text{robust}} = \argmin_{h \in \Hspace}
    \agg_{Q \in \U(P_{\mathrm{ref}})}
    \mathbb{E}_{(x,y)\sim Q} \big[ \ell(h(x), y) \big]
    \label{eq:robust_framework}
\end{equation}
This objective recovers standard robustness methods as special cases.
ERM follows from $\Pref=\Pemp$ and $\varrho=0$; methods such as label smoothing,
VRM, and Mixup modify the reference distribution while keeping a singleton
ambiguity set; DRO/DFO retain $\Pref=\Pemp$ while varying the
disambiguation principle over a non-trivial ambiguity set.
See Table~\ref{tab:special-cases} for an overview.

Beyond recovering existing methods, the framework also makes their composition explicit. Different robustness mechanisms may act on different parts of the learning problem and can therefore be combined within a common objective. In particular, let uncertainty decompose across input space, label space, and the sample distribution, each equipped with its own perturbation model and strength. The disambiguation principle may then be written as $\aggX \circ \aggY \circ \aggS,$ allowing different stances toward different sources of uncertainty.

\subsection{The Case of Linear Least Squares}
\label{sec:linear_interactions}

Let \(h_w(x)=w^\top x\) be a linear predictor trained with squared loss \(\ell(h_w(x),y)=(y-w^\top x)^2\). This covers linear regression and, for \(y\in\{-1,+1\}\), least-squares classification. We include an unpenalized intercept and write \(w=(b,\beta)\), with
perturbations acting only on the features.
For several robustness techniques, this setting yields simple closed-form
reductions that allow us to study their interactions. We write \(R_{\mathrm{ERM}}\), \(R_{\mathrm{VRM}}\), etc.\ for the empirical risks of the respective objectives.
\begin{figure*}[t]
  \centering
  \resizebox{\linewidth}{!}{\begin{tikzpicture}[font=\fontsize{8}{9.5}\selectfont]

\tikzset{
  lhead/.style={draw, thick, rounded corners=2.5pt, align=center,
                text width=2.95cm, minimum height=0.95cm, inner sep=3pt,
                font=\fontsize{8}{9.5}\selectfont},
  lio/.style={lhead, text width=1.2cm, inner sep=1.5pt, draw=clGrayBorder, fill=clGray},
  lcell/.style={rounded corners=1.5pt, align=center, text width=3.16cm,
                minimum height=0.8cm, inner sep=0pt, font=\fontsize{8}{9.5}\selectfont},
  lstance/.style={anchor=east, font=\scriptsize\itshape, text=black!55},
}

\def\xB{0}     \def\xA{2.593} \def\xC{6.113}
\def\xD{9.633} \def\xE{13.153} \def\xF{15.746}
\def\yP{-1.05} \def\yN{-1.95} \def\yO{-2.85}

\node[lio]  (batch) at (\xB,0) {\textbf{Batch}\\[1pt]$\sim\Pemp$};
\node[lhead, fill=clGreen,  draw=clGreenBorder]  (s1) at (\xA,0)
  {\textbf{1.\ Reference}\\[1pt]$\Pref$: Mixup, VRM, $\sigma$};
\node[lhead, fill=clOrange, draw=clOrangeBorder] (s2) at (\xC,0)
  {\textbf{2.\ Inputs}\\[1pt]$\aggX$: $L_2$ ball, $\varrho$};
\node[lhead, fill=clPurple, draw=clPurpleBorder] (s3) at (\xD,0)
  {\textbf{3.\ Labels}\\[1pt]$\aggY$: credal set, $\alpha$};
\node[lhead, fill=clBlue,   draw=clBlueBorder]   (s4) at (\xE,0)
  {\textbf{4.\ Aggregation}\\[1pt]$\aggS$: KL, strength $s$};
\node[lio]  (upd) at (\xF,0) {\textbf{Update}\\[1pt]$\nabla_\theta\mathcal{L}$};

\foreach \a/\b in {batch/s1, s1/s2, s2/s3, s3/s4, s4/upd}
  \draw[arr] (\a.east) -- (\b.west);

\node[lstance] at (batch.east |- 0,\yP) {pessimistic};
\node[lstance] at (batch.east |- 0,\yN) {neutral};
\node[lstance] at (batch.east |- 0,\yO) {optimistic};

\foreach \y in {\yP,\yN}
  \draw[black!12] (batch.west |- 0,\y-0.45) -- ([xshift=0.1cm]s4.east |- 0,\y-0.45);

\node[lcell, fill=clGreenLight, draw=clGreenBorder!50, dashed,
      minimum height=2.6cm] at (\xA,\yN)
  {\textit{Mixup}\\[1pt]
   $\tilde{x}=\lambda x_i+(1{-}\lambda)x_j$\\[1pt]
   $\lambda\sim\mathrm{Beta}(a,a)$\\[5pt]
   \textit{then Gaussian VRM}\\[1pt]
   $\tilde{x}\leftarrow\tilde{x}+\sigma\varepsilon$\\[1pt]
   $\varepsilon\sim\mathcal{N}(0,I)$};

\node[lcell, fill=clOrangeLight] at (\xC,\yP)
  {\textit{W-DRO}\\[1pt]$\max_{\|\delta\|\le\varrho}\ell(h(\tilde{x}{+}\delta),y)$};
\node[lcell, fill=clOrangeLight] at (\xC,\yN)
  {\textit{none}\\[1pt]$\delta=0$};
\node[lcell, fill=clOrangeLight] at (\xC,\yO)
  {\textit{W-DFO}\\[1pt]$\min_{\|\delta\|\le\varrho}\ell(h(\tilde{x}{+}\delta),y)$};

\node[lcell, fill=clPurpleLight] at (\xD,\yP)
  {\textit{adversarial label shift}\\[1pt]$\alpha$ on $\arg\min_{k\ne y}p_k$};
\node[lcell, fill=clPurpleLight] at (\xD,\yN)
  {\textit{label smoothing}\\[1pt]$(1{-}\alpha)\,t+\alpha\mathbf{1}/K$};
\node[lcell, fill=clPurpleLight] at (\xD,\yO)
  {\textit{label relaxation}\\[1pt]zero loss in credal set};

\node[lcell, fill=clBlueLight] at (\xE,\yP)
  {\textit{KL-DRO, Group DRO}\\[1pt]$\tau\log\sum_g p_g e^{\ell_g/\tau}$};
\node[lcell, fill=clBlueLight] at (\xE,\yN)
  {\textit{ERM}\\[1pt]$\sum_g p_g\ell_g$};
\node[lcell, fill=clBlueLight] at (\xE,\yO)
  {\textit{KL-DFO}\\[1pt]$-\tau\log\sum_g p_g e^{-\ell_g/\tau}$};

\end{tikzpicture}}
  \caption{Unified robust training procedure. A mini-batch passes through four stages: distributional enrichment ($P_{\mathrm{ref}}$), input-space perturbation ($\aggX$), label-space perturbation ($\aggY$), and batch-level aggregation ($\aggS$), before a single gradient update.}
  \label{fig:robust_pipeline}
\end{figure*}
Vicinal risk minimization (VRM) with Gaussian vicinity \(\tilde x=x+\epsilon\), \(\epsilon\sim\mathcal N(0,\sigma^2 I)\) adds a ridge penalty:
\begin{equation}
R_{\mathrm{VRM}}(w)=R_{\mathrm{ERM}}(w)+\sigma^2\|\beta\|_2^2.
\tag{Props.~\ref{prop:special_case_vrm_regression}, \ref{prop:special_case_vrm_classification}}
\end{equation}
For Mixup, with \(\bar x=\frac1n\sum_i x_i\), \(\bar y=\frac1n\sum_i y_i\), and \(c_\alpha=\mathbb E[\lambda(1-\lambda)]\in(0,\tfrac12)\) for \(\lambda\sim\operatorname{Beta}(\alpha,\alpha)\),
\begin{align}
R_{\mathrm{Mixup}}(w) ={}& (1-2c_\alpha)\,R_{\mathrm{ERM}}(w) \nonumber  \\
& \qquad + 2c_\alpha\,(\bar y-w^\top\bar x)^2. \tag{Props.~\ref{prop:special_case_mixup_regression}, \ref{prop:special_case_mixup_classification}}
\end{align}
At any least-squares optimum, the residuals have zero empirical mean, \(w^\top\bar x=\bar y\), so the additional term vanishes and Mixup and ERM share the same minimizer. Label smoothing with strength \(\alpha_{\mathrm{LS}}\in[0,1)\) replaces \(y_i\) by \(\tilde y_i=(1-\alpha_{\mathrm{LS}})y_i\) and hence regularizes by creating a shrinkage effect \citep{copas1983regression}:
\begin{equation}
R_{\mathrm{LS}}(w)=(1-\alpha_{\mathrm{LS}})^2\,R_{\mathrm{ERM}}\!\left(\frac{w}{1-\alpha_{\mathrm{LS}}}\right).
\tag{Prop.~\ref{prop:special_case_ls_classification}}
\end{equation}
Wasserstein DRO with radius \(\varrho\) and transport cost \(\|\cdot\|_q\) admits the small-radius expansion
\begin{align}
R_{\mathrm{DRO}}(w) ={}& R_{\mathrm{ERM}}(w)
\tag{Props.~\ref{prop:special_case_dro_regression}, \ref{prop:special_case_dro_classification}} \\
& \; \; + 2\varrho\,\|\beta\|_{q^\ast}\,\frac1n\sum_{i=1}^n|y_i-w^\top x_i|+o(\varrho) \nonumber ,
\end{align}
where \(\|\cdot\|_{q^\ast}\) is the dual norm of \(\|\cdot\|_q\).

These reductions show that the techniques act on the least-squares problem in qualitatively different ways: VRM adds quadratic regularization, label smoothing rescales the solution, Wasserstein DRO adds a residual-weighted norm penalty, and Mixup leaves the minimizer unchanged. They also allow us to study what happens under composition. Specifically, two behaviours emerge. We call a composition \emph{redundant} if its effect can be reproduced by retuning one of its constituent techniques, and \emph{non-redundant} otherwise. The results below illustrate both cases. Proofs are given in the appendix.

\begin{proposition}[Mixup and VRM are redundant up to reparameterization]
\label{prop:mixup_vrm_reparam}
Let Mixup with parameter \(c_\alpha\) be applied to the training samples, followed by Gaussian VRM with strength \(\gamma=\sigma^2\) on the mixed samples. Then
\[
w^\star_{\mathrm{Mixup}+\mathrm{VRM}}(c_\alpha,\gamma)
=
w^\star_{\mathrm{VRM}}
\left(
\frac{\gamma}{1-2c_\alpha}
\right).
\]
Thus, composing Mixup with VRM does not enlarge the VRM solution family, but changes its effective regularization strength by the factor $\frac{1}{1-2c_\alpha}>1.$
\end{proposition}

\begin{corollary}[Mixup and label smoothing are redundant]
\label{cor:mixup_ls}
In least-squares classification,
\[
w^\star_{\mathrm{Mixup}+\mathrm{LS}}
=
w^\star_{\mathrm{LS}}
=
(1-\alpha_{\mathrm{LS}})
w^\star_{\mathrm{ERM}}.
\]
Hence, adding Mixup does not enlarge the solution family obtained by label smoothing.
\end{corollary}

\begin{proposition}[Label smoothing and VRM are non-redundant]
\label{prop:ls_vrm}
Let \(a=1-\alpha_{\mathrm{LS}}\in(0,1)\) and \(\gamma>0\). Then
\[
w^\star_{\mathrm{LS}+\mathrm{VRM}}
(\alpha_{\mathrm{LS}},\gamma)
=
a\,w^\star_{\mathrm{VRM}}(\gamma).
\]
Under mild non-degeneracy conditions on the data, this solution cannot be obtained by retuning either VRM or label smoothing alone. Hence, their composition is non-redundant.
\end{proposition}

Together, these results illustrate two qualitatively different interactions. Mixup does not introduce an additional solution dimension when combined with either VRM or label smoothing, although its composition with VRM can change the effective regularization strength depending on the order of the operations. By contrast, label smoothing and VRM can generate solutions outside the VRM solution path. As we show next, these redundancy relations are specific to the linear least-squares setting and need not persist under other losses or hypothesis classes.

\subsection{Beyond Linear Least Squares}
\label{sec:beyond_linear_ls}

The exact redundancies above rely on the linear-quadratic structure of the learning problem and should therefore not be taken as universal properties of Mixup, VRM, or label smoothing. We illustrate this using Mixup, whose redundancy under linear least squares provides a particularly transparent example. Specifically, we show that the redundancy can disappear once either the squared loss or the linearity in the input is removed.

\begin{proposition}[Mixup redundancy is specific to linear least squares]
\label{prop:mixup_counterexamples}
\textup{(i)} There is a binary dataset on which a linear logistic model trained with cross-entropy satisfies \(w^\star_{\mathrm{Mixup}}\neq w^\star_{\mathrm{ERM}}\). \textup{(ii)} There is a dataset and a predictor \(h_w(x)=b+wx^2\), linear in its parameters but nonlinear in the input, trained with squared loss and intercept, for which \(w^\star_{\mathrm{Mixup}}\neq w^\star_{\mathrm{ERM}}\).
\end{proposition}

Thus, the redundancy of robustness techniques depends on the loss and hypothesis class,
which makes the appropriate choice and combination of such techniques inherently problem-dependent.

\section{Tractable Learning and Optimization}
\label{sec:algorithm}

The unified framework \eqref{eq:robust_framework} encompasses objectives with fundamentally different structures, ranging from inner max/min problems over distribution spaces to sampling-based vicinal objectives. In this generallity, it is clearly not tractable. We therefore derive a tractable realization that
implements its different components within a single differentiable training
pipeline. This exposes the corresponding choices of robustification techniques and their strengths as
tunable hyperparameters, enabling end-to-end training as well as data-driven
selection and composition through standard HPO.

\begin{table*}[!t]
\centering
\caption{\textbf{Q1: Individual robustness methods on CIFAR-10-C.}
ResNet-18 trained end to end; error, Brier score, and ECE on held-out corruption groups
(severity~3, lower is better) under three selection rules, averaged over three lattices.
Best \textbf{bold}, second \underline{underlined}.}
\label{tab:cifar_q1}

\small
\setlength{\tabcolsep}{3pt}
\begin{tabular}{l*{9}{>{\centering\arraybackslash}m{1.35cm}}}
\toprule
& \multicolumn{3}{c}{ID validation}
& \multicolumn{3}{c}{OOD validation, out-of-sample}
& \multicolumn{3}{c}{OOD validation, in-sample} \\
\cmidrule(lr){2-4} \cmidrule(lr){5-7} \cmidrule(lr){8-10}
& Error $\downarrow$ & Brier $\downarrow$ & ECE $\downarrow$
& Error $\downarrow$ & Brier $\downarrow$ & ECE $\downarrow$
& Error $\downarrow$ & Brier $\downarrow$ & ECE $\downarrow$ \\
\midrule
ERM
& $.251$ & $.401$ & $.159$
& $.251$ & $.401$ & $.159$
& $.251$ & $.401$ & $.159$ \\

VRM
& $.250$ & $.401$ & $.159$
& $.224$ & $.357$ & $.138$
& $\underline{.209}$ & $.331$ & $.126$ \\

Mixup
& $\underline{.230}$ & $\underline{.367}$ & $\mathbf{.107}$
& $\underline{.221}$ & $\underline{.316}$ & $\underline{.098}$
& $.218$ & $\underline{.312}$ & $\mathbf{.046}$ \\

W-DRO
& $\mathbf{.195}$ & $\mathbf{.308}$ & $\underline{.116}$
& $\mathbf{.167}$ & $\mathbf{.246}$ & $\mathbf{.081}$
& $\mathbf{.142}$ & $\mathbf{.219}$ & $\underline{.073}$ \\

KL-tilt
& $.260$ & $.406$ & $.142$
& $.253$ & $.395$ & $.141$
& $.244$ & $.384$ & $.134$ \\

LS
& $.258$ & $.398$ & $.140$
& $.253$ & $.381$ & $.113$
& $.245$ & $.377$ & $.112$ \\
\bottomrule
\end{tabular}
\end{table*}

\begin{table}[!t]
\centering
\caption{\textbf{Q1: Individual robustness methods on HH-RLHF.}
Bradley--Terry model on frozen embeddings, trained on
\texttt{helpful-base} and evaluated on \texttt{harmless-base}
(error and ECE, lower is better).}
\label{tab:reward_q1}
{\small
\setlength{\tabcolsep}{2.5pt}
\begin{tabular}{lcccc}
\toprule
& \multicolumn{2}{c}{ID val.}
& \multicolumn{2}{c}{OOD val., in-s.} \\
\cmidrule(lr){2-3} \cmidrule(lr){4-5}
& Err.\ $\downarrow$ & ECE $\downarrow$
& Err.\ $\downarrow$ & ECE $\downarrow$ \\
\midrule
ERM     & $.490$ & $.081$ & $.490$ & $.081$ \\
VRM     & $\mathbf{.388}$ & $.047$ & $\mathbf{.386}$ & $\mathbf{.024}$ \\
Mixup   & $.461$ & $\mathbf{.029}$ & $.400$ & $\underline{.027}$ \\
W-DRO   & $\underline{.401}$ & $\underline{.044}$ & $\underline{.396}$ & $.031$ \\
KL-tilt & $.492$ & $.071$ & $.492$ & $.067$ \\
LS      & $.480$ & $.067$ & $.464$ & $\mathbf{.024}$ \\
LR      & $.483$ & $.073$ & $.459$ & $\mathbf{.024}$ \\
\bottomrule
\end{tabular}}
\end{table}

A mini-batch from $\Pemp$ passes through four stages before the gradient update (Figure~\ref{fig:robust_pipeline}).
\textbf{Stage~1} sets $\Pref$ by enriching the batch, e.g., with VRM (bandwidth $\sigma$) and Mixup.
\textbf{Stage~2} realizes $\aggX$ by perturbing each input within a Wasserstein ball of radius $\varrho$ through projected gradient steps~\citep{madry2018towards,sinha2018certifying}.
\textbf{Stage~3} realizes $\aggY$ by perturbing the labels within a credal set of mass $\alpha$, from label smoothing (neutral) through label relaxation (optimistic) to adversarial label shift (pessimistic;
Appendix~\ref{app:robust_label_loss}).
\textbf{Stage~4} realizes $\aggS$ by the closed-form dual of the KL-constrained worst- or best-case expectation~\citep{ben-talRobustSolutionsOptimization2013,levyLargescaleMethodsDistributionally2020}, $\pm\tau\log\sum_g p_g\exp(\pm\ell_g/\tau)$, over the cells $g$ of a partition of the batch, with mean losses $\ell_g$ and reference weights $p_g$.
With one cell per sample this is tilted ERM~\citep{term2021}; on annotated groups, its pessimistic limit $\tau\to0$ is Group DRO~\citep{Sagawa2020Distributionally}.
We set $\tau$ through a strength $s\in[0,1]$ defined by the effective sample size of the tilted weights (Appendix~\ref{app:implementation}).
Stages~2--4 each take a pessimistic, neutral, or optimistic stance, and every stage can be switched off, so that ERM and each single technique are special cases.

\textbf{Data-driven configuration.}
Because every design choice is exposed as a hyperparameter of a single objective, configuring an entire robust learning pipeline reduces to a single hyperparameter optimization problem that can be solved with standard black-box optimization methods. This reframes the practitioner's role from selecting a robustness technique a priori to specifying a search budget. Rather than committing in advance to a particular approach (e.g., W-DRO with a hand-tuned radius), one performs joint HPO over $(\sigma, \varrho, \alpha, \tau)$. This unified formulation
naturally includes existing techniques as special cases while also enabling
novel combinations.
Thus, rather than expecting the practitioner to anticipate which combination of robustness techniques works best, this decision is now made in a data-driven way.

Let us note that the design space could of course be extended by additional dimensions. For example, focusing on techniques to robustify learning, we assume the hypothesis space $\mathcal{H}$ and loss $\ell$ to be be given. In other words, we assume that the practitioner has already decided about the type of model to be learned and the loss to be minimized. If this is not the case, e.g., because the practitioner is agnostic with regard to the type of model, $\mathcal{H}$ could be included as an additional configuration parameter.

\section{Experiments}
\label{sec:experiments}

Because all robustness methods are implemented within the same training
procedure (Section~\ref{sec:algorithm}), we can compare them under identical
training conditions and also evaluate combinations of methods, which we refer
to as \emph{coalitions}. In the following, we aim to answer three research questions:
\textbf{Q1:} Does any single robustness method perform consistently well across different settings?
 \textbf{Q2:} Does combining robustness methods improve performance?
\textbf{Q3:} How do preferred compositions vary across shifts and target metrics?

\textbf{Setup.} We evaluate our framework on image classification under corruption on CIFAR-10-C
\citep{hendrycksBenchmarkingNeuralNetwork2019} and preference learning under
distribution shift on HH-RLHF~\citep{bai2022training}. On CIFAR-10-C, we train
a ResNet-18 on CIFAR-10 and evaluate on the 19 corruption families at
severity~3. On HH-RLHF, we train a Bradley--Terry model
\citep{bradley_rank_1952} on \texttt{helpful-base} and evaluate on
\texttt{harmless-base}. We additionally study group robustness and sample
re-weighting on Waterbirds~\citep{Sagawa2020Distributionally}; these results
are reported in Appendix~\ref{app:waterbirds}.

We consider the robustness methods VRM and Mixup (stage~1), Wasserstein DRO (W-DRO, stage~2), label smoothing (LS) or label relaxation (LR) (stage~3), and the KL tilt of stage~4 (KL-tilt).
Each method alone is a baseline against ERM, while \textsc{Composition} selects among coalitions by the same rule.
All coalitions are evaluated on the same quasi-random (Sobol) points~\citep{owen-sig03a}, restricted to their active techniques, and thus receive equal tuning budgets. The hyperparameter ranges were selected in a data-driven manner on in-distribution validation data (cf. Appendix~\ref{app:search_space}).
We report classification error, Brier score, and expected calibration error (ECE) (cf. Appendix~\ref{app:protocol}).
\begin{table}[!t]
\centering
\caption{\textbf{Q2: Single-method selection versus composition on CIFAR-10-C.}
Results under three validation regimes; lower is better.}
\label{tab:cifar_q2}

\small
\setlength{\tabcolsep}{5pt}
\begin{tabular}{llc>{\columncolor{gray!12}}c}
\toprule
Selection rule & Metric & Best single & \textsc{Composition} \\
\midrule
\multirow{3}{*}{ID val.}
& Error $\downarrow$ & .244 & \textbf{.240} \\
& Brier $\downarrow$ & .376 & \textbf{.284} \\
& ECE $\downarrow$   & \textbf{.107} & .110 \\
\midrule

\multirow{3}{*}{\makecell[l]{OOD val.\\out-of-s.}}
& Error $\downarrow$ & .167 & \textbf{.148} \\
& Brier $\downarrow$ & .246 & \textbf{.213} \\
& ECE $\downarrow$   & .101 & \textbf{.024} \\
\midrule

\multirow{3}{*}{\makecell[l]{OOD val.\\in-s.}}
& Error $\downarrow$ & .140 & \textbf{.122} \\
& Brier $\downarrow$ & .216 & \textbf{.187} \\
& ECE $\downarrow$   & .043 & \textbf{.022} \\
\bottomrule
\end{tabular}
\end{table}

\textbf{Selection rules.}
Following \citet{gulrajani-lost-dg-21}, we compare different selection
rules. Each selection rule chooses one configuration, i.e.\ one coalition and one
Sobol point, from the corresponding candidate pool. Every rule selects
separately for each reported metric, on the value of that metric. The
available selection rules depend on the benchmark.

For CIFAR-10-C, \emph{ID validation} selects on the clean validation split.
For \emph{out-of-sample OOD validation}, the 19 corruption families are
partitioned into four groups (noise, blur, weather, and digital); when
evaluating one group, selection is performed on the remaining three.
\emph{In-sample OOD validation} instead selects directly on held-out data from
the evaluated corruption group and therefore serves as an oracle reference. For HH-RLHF, \emph{ID validation} selects on the \texttt{helpful-base}
validation split, while \emph{in-sample OOD validation} selects on one half of
\texttt{harmless-base} and reports on the other half. Because HH-RLHF contains
only one shifted environment, no out-of-sample OOD selection rule is available.

\begin{table}[!t]
\centering
\caption{\textbf{Q2: Single-method selection versus composition on HH-RLHF.}
Results under the same model-selection rules
(error and ECE; lower is better).}
\label{tab:reward_q2}
{\small
\setlength{\tabcolsep}{2.5pt}
\begin{tabular}{lcccc}
\toprule
& \multicolumn{2}{c}{ID val.}
& \multicolumn{2}{c}{OOD val., in-s.} \\
\cmidrule(lr){2-3} \cmidrule(lr){4-5}
& Err.\ $\downarrow$ & ECE $\downarrow$
& Err.\ $\downarrow$ & ECE $\downarrow$ \\
\midrule
Single, selected
& $.480$ & $.047$
& $.386$ & $.027$ \\

\rowcolor{gray!12}
\textsc{Composition}
& $\mathbf{.474}$ & $\mathbf{.039}$
& $\mathbf{.321}$ & $\mathbf{.020}$ \\

\bottomrule
\end{tabular}}
\end{table}

\subsection{Q1: Does any single robustness method perform consistently well across different settings?}
\label{sec:q1}

On CIFAR-10-C, W-DRO is the strongest individual method for error and Brier score across all selection regimes, while calibration favors either Mixup or W-DRO depending on the selection rule (Table~\ref{tab:cifar_q1}). On HH-RLHF, VRM performs best for error under both ID and OOD selection, whereas Mixup achieves the best ID calibration and several methods perform similarly under OOD calibration (Table~\ref{tab:reward_q1}). KL-tilt provides little improvement on either benchmark. Moreover, the ranking varies even within a benchmark: across CIFAR-10-C corruption families, individual methods improve some shifts while degrading others (Appendix~\ref{app:per_corruption}). Overall, the preferred robustness method depends on the benchmark, target metric, and selection rule, with no single method dominating across all settings.

\begin{figure}[!t]
\centering
\includegraphics[width=0.85\linewidth]{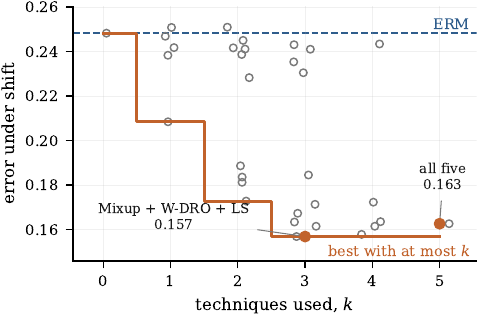}
\caption{CIFAR-10-C coalition analysis at severity~3, averaged over 19 corruption families.
Best shifted error among coalitions of size $\le k$; grey points show individual coalitions.}
\label{fig:cifar_lattice}
\end{figure}

\subsection{Q2: Does combining robustness methods improve performance?}
\label{sec:q2}

To assess whether composition improves performance, we compare the configuration selected from all coalitions with the best single technique selected under the same validation rule. Overall, allowing compositions improves performance across both datasets and selection regimes, with the largest gains under OOD validation. On CIFAR-10-C, composition improves error and Brier score under all three selection rules and substantially improves calibration under OOD validation, while ID-selected ECE is slightly worse (Table~\ref{tab:cifar_q2}). HH-RLHF shows the same overall pattern: the selected composition improves both error and calibration under ID and OOD validation, with particularly large gains under OOD selection (Table~\ref{tab:reward_q2}). These results indicate that the larger design space can be beneficial even under ID selection, but becomes especially valuable when the validation signal reflects the target distribution shift.

The benefit does not simply increase with coalition size, as can be seen in Figure~\ref{fig:cifar_lattice}. On CIFAR-10-C, performance improves from single methods to small coalitions, reaching its minimum error with three methods and not improving further with larger coalitions. Thus, gains come from selecting suitable combinations rather than simply combining more robustness methods.

\subsection{Q3: How do preferred compositions vary across shifts and target metrics?}
\label{sec:q3}

\begin{figure}[t]
\centering
\includegraphics[width=\linewidth]{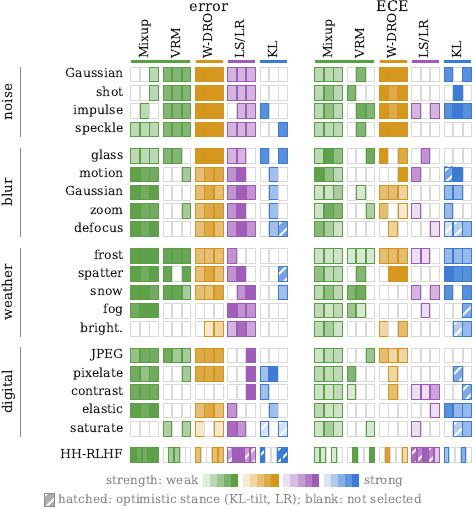}
\caption{Combined robustness methods selected by \textsc{Composition} for each shift, on error (left) and ECE (right).
Rows are the 19 corruption families of CIFAR-10-C at severity~3 and the HH-RLHF dataset.}
\label{fig:q3_family_picks}
\end{figure}

Figure~\ref{fig:q3_family_picks} reveals systematic patterns in the selected compositions. For CIFAR-10-C error, W-DRO is selected for most corruption families, while the remaining components vary substantially across shifts. For calibration, the selected configurations are more
heterogeneous, with KL-tilt and label-based interventions appearing more
prominently. Importantly, the preferred composition often changes between
error and ECE even for the same corruption family. A similar metric dependence
is visible on HH-RLHF. Thus, there is no single best coalition; instead,
different robustness methods play different roles depending on setting.

\section{Conclusion and Future Work}
\label{sec:conclusion}

We introduced a unified framework for robust supervised learning that organizes a broad class of existing methods along three design axes: the reference
distribution $P_{\mathrm{ref}}$, the ambiguity set $\U(P_{\mathrm{ref}})$, and
the disambiguation principle $\agg$. Within this framework, we analyzed how robustness techniques interact: in linear least squares, some compositions are redundant up to
reparameterization whereas others yield distinct solution families, and these
relations need not persist under different losses or hypothesis classes. We further developed a modular training procedure that exposes the design choices as hyperparameters. Empirically, the best robustness method varies across tasks, target metrics, and available validation information, while suitably selected compositions often outperform the best single method. Together, these results support selecting and composing robustness methods in a data-driven manner within a common design space.

Several directions remain for future work. First, extending the framework to more general aggregation operators, e.g. the Choquet integral \citep{choquetTheoryCapacities1954}, may further increase expressivity.
Second, the framework naturally lends itself to downstream applications, such as learning credal ensembles via different instantiations of its components~\citep{wangLearningCredalEnsembles2026}.

\textbf{Limitations and Broader Impact.}\label{sec:limitations}
Our empirical study focuses on a representative set of robustness methods and benchmarks, so the observed interaction patterns may not transfer uniformly to all tasks or model classes. Likewise, our exact theoretical characterizations are derived for simplified settings.

We do not foresee societal risks beyond those generally associated with improved learning and model-selection techniques.

\bibliographystyle{plainnat}
\bibliography{literature}

\clearpage
\appendix

\onecolumn

\section{Extended Background}

In supervised learning, we consider an input space $\mathcal{X}$,
an output space $\mathcal{Y}$ and we assume there exists an unknown distribution
$P^\star \in \mathcal{P}(\mathcal{X} \times \mathcal{Y})$,
where $\mathcal{P}(\mathcal{X} \times \mathcal{Y})$ denotes the set of all probability measures on $\mathcal{X} \times \mathcal{Y}$. Given a hypothesis class
$\Hspace$ and a loss function $\ell : \mathcal{Y} \times
\mathcal{Y} \to \mathbb{R}$, the goal is to find a hypothesis
$h \in \Hspace$ minimizing the population risk,
\[
h^\star
= \argmin_{h \in \Hspace} R(h; P^\star)
= \argmin_{h \in \Hspace} \mathbb{E}_{(x,y)\sim P^\star}
[\ell(h(x), y)].
\]
Since $P^\star$ is unknown, it cannot be optimized against
directly. Instead, given a finite dataset
$\mathcal{D} = \{(x_i, y_i)\}_{i=1}^n$ drawn i.i.d.\ from
$P^\star$, one constructs the empirical distribution
\(
\Pemp := \frac{1}{n}\sum_{i=1}^n \delta_{(x_i,y_i)},
\)
where $\delta_{(x,y)}$ denotes the Dirac measure at $(x,y)$, and
uses it as a proxy for $P^\star$. Empirical risk
minimization (ERM) then selects a hypothesis by minimizing the
empirical risk,
\[
\hat{h}_{\mathrm{ERM}}
= \argmin_{h \in \Hspace} R(h; \Pemp)
= \argmin_{h \in \Hspace} \frac{1}{n}\sum_{i=1}^n
\ell(h(x_i), y_i).
\]
ERM can be suboptimal in practice for several
reasons~\citep{blanchetDistributionallyRobustOptimization2025}.
First, $\Pemp$ is a poor
proxy for $P^\star$ when data are scarce: being a degenerate
discrete measure, it is sensitive to random fluctuations and
assigns zero probability to unobserved regions of the input
space. Second, the observed samples may not exactly follow
$P^\star$ due to data contamination such as measurement noise
or label errors. Third, model performance is ultimately
evaluated under a deployment distribution
$P_{\mathrm{deploy}}$ that may differ from $P^\star$ due to
e.g. distributional shift. Together, these issues motivate
learning paradigms that go beyond ERM, and we now recall
existing techniques relevant to this work, organized by the
failure mode they address.

\subsection{Finite-sample degeneracy}

Techniques in this subsection address the degeneracy of the empirical
distribution $\Pemp$ arising from finite data. Rather than
optimizing directly over $\Pemp$, they enrich or smooth the
training distribution to improve generalization to unseen regions
of the input space.

\paragraph{Vicinal Risk Minimization (VRM).}
VRM~\citep{chapelleVicinalRiskMinimization2000} defines a
\emph{vicinal distribution} that spreads probability mass over
neighborhoods of each training example:
\(
P_{\text{vic}}
= \frac{1}{n} \sum_{i=1}^n \nu(\cdot \mid x_i, y_i),
\)
where $\nu(\cdot \mid x_i, y_i)$ is a vicinity distribution
centered at $(x_i, y_i)$. The VRM objective is then
\[
\hat{h}_{\mathrm{VRM}}
= \argmin_{h \in \Hspace}
\mathbb{E}_{(x,y)\sim P_{\text{vic}}}[\ell(h(x),y)].
\]

\paragraph{Mixup.}
Mixup~\citep{zhangMixupEmpiricalRisk2018} constructs synthetic
training samples via convex interpolation between pairs of
training examples. Given two examples $(x_i, y_i)$ and
$(x_j, y_j)$ with one-hot encoded labels, it constructs
$\tilde{x} = \lambda x_i + (1-\lambda)x_j$ and
$\tilde{y} = \lambda y_i + (1-\lambda)y_j$, where
$\lambda \sim \mathrm{Beta}(\alpha, \alpha)$. Unlike VRM, which
spreads mass locally around individual points, Mixup explicitly
interpolates between distinct examples, encouraging the model
to behave approximately linearly between training points.

\subsection{Label noise and miscalibration}

Techniques in this subsection address imperfections in the observed
labels. Rather than treating labels as ground truth, they
either soften target distributions or replace individual labels
with sets of admissible distributions, leading to better-calibrated
and more robust predictions.

\paragraph{Label Smoothing.}
Label smoothing~\citep{szegedyRethinkingInceptionArchitecture2016}
uniformly redistributes a small amount of probability mass from
the true class to all other classes. For a $K$-class problem with
one-hot label $y = (0,\ldots,0,1,0,\ldots,0) \in \mathbb{R}^K$,
it constructs
\(
\tilde{y}
=(
\alpha/K, \ldots, \alpha/K,
\; 1-\alpha+\alpha/K,
\; \alpha/K, \ldots, \alpha/K
),
\)
where $\alpha \in [0,1]$ controls the smoothing strength.

\paragraph{Label Relaxation.}
Label relaxation~\citep{DBLP:conf/aaai/LienenH21} takes a
set-valued perspective: rather than committing to a single
softened target, it replaces each observed label $y_i$ with an
ambiguity set of admissible label distributions,
\(
\mathcal{Q}_i^\alpha
=\{
p \in \mathcal{P}(\mathcal{Y})
\;|\;
\sum_{y \neq y_i} p(y) \le \alpha
\},
\)
where $\alpha \in [0,1]$ controls the degree of relaxation. Any
prediction within $\mathcal{Q}_i^\alpha$ incurs zero loss;
predictions outside are penalized by their minimal distance to
the set.

\subsection{Distribution shift}

The techniques in this subsection address the discrepancies between the
training distribution $\Pemp$ and the deployment distribution
$P_{\mathrm{deploy}}$. They do so by optimizing not against
$\Pemp$ directly, but against a set of distributions in its
vicinity, an \emph{ambiguity set} $\U(\Pemp)$, differing in how they aggregate over this set: pessimistically
(DRO) or optimistically (DFO).

\paragraph{Distributionally Robust Optimization (DRO).}
DRO~\citep{blanchetDistributionallyRobustOptimization2025}
optimizes against the worst-case distribution within the
ambiguity set. The ambiguity set is typically defined via a
statistical distance measure; common choices include
the KL-divergence, or the Wasserstein
distance. In the latter instantiation,
\(
\U(\Pemp)
= \{Q \in \mathcal{P}(\mathcal{X} \times \mathcal{Y})
: W_p(Q, \Pemp) \le \varrho \},
\)
where $W_p$ denotes the $p$-Wasserstein distance and $\varrho \ge 0$
controls the size of the ambiguity set. The DRO objective is then
\[
\hat{h}_{\mathrm{DRO}}
= \argmin_{h\in\Hspace}
\sup_{Q \in \U(\Pemp)}
\mathbb{E}_{(x,y)\sim Q}[\ell(h(x),y)].
\]
Notably, adversarial training~\citep{madry2018towards} can
be viewed as a special case of Wasserstein-DRO in which the
ambiguity set is restricted to pointwise input perturbations
$x \mapsto x + \delta$ within a geometric set
$\Delta \subset \mathbb{R}^d$~\citep{sinha2018certifying}.

\paragraph{Distributionally Favorable Optimization (DFO).}
DFO~\citep{jiangDistributionallyFavorableOptimization2024a} takes
the opposite stance, optimizing for the best-case distribution
within the ambiguity set:
\[
\hat{h}_{\mathrm{DFO}}
= \argmin_{h\in\Hspace}
\inf_{Q \in \U(\Pemp)}
\mathbb{E}_{(x,y)\sim Q}[\ell(h(x),y)].
\]
This is motivated by settings where observed data may be corrupted
or pessimistically biased, and optimizing for the best-case risk
aims to recover a cleaner signal.

\paragraph{Probabilistically Robust Learning (PRL).}
PRL~\citep{proba_robust_robey22} interpolates between DRO and ERM
by requiring robustness to \emph{most} perturbations rather than
all, addressing the tendency of worst-case techniques to
overcompensate for rare events. Given a distribution $r$ over
the perturbation set $\Delta$ and a tolerance level
$\rho \in [0,1)$, the PRL objective is
\[
\hat{h}_{\mathrm{PRL}}
= \argmin_{h \in \Hspace}\,
\mathbb{E}_{(x,y)\sim \Pemp}
\bigl[\,\rho\text{-}\operatorname{ess\,sup}_{\delta \sim r}\,
\ell(h(x+\delta), y)\bigr],
\]
where the $\rho$-essential supremum discards the worst fraction
$\rho$ of perturbations. As $\rho \to 0$, PRL recovers adversarial
training; as $\rho \to 1$, it approaches ERM.

\section{Conceptual Extension} \label{app:coneptual}

\subsection{Robust Label Loss}
\label{app:robust_label_loss}
In order to incorporate a pessimistic variant of label relaxation, we adapt the loss function proposed by \citet{DBLP:conf/aaai/LienenH21}.
The credal target set for the optimistic case is given by
$$Q^\alpha_i = \Big\{ p \in \mathcal{P}(\mathcal{Y}) \;\Big|\; \textstyle\sum_{y \neq y_i} p(y) \leq \alpha \Big\},$$
and their optimistic loss is
$$L^\text{opt}(Q^\alpha_i, \hat{p}) = \min_{p \in Q^\alpha_i} \KL(p \| \hat{p}).$$
The natural pessimistic counterpart replaces the infimum with a supremum.
Instantiating the supremum with KL yields
$$\sup_{p \in Q^\alpha_i} \KL(p \| \hat{p}) = (1-\alpha) \log \tfrac{1-\alpha}{\hat{p}(y_i)} + \alpha \log \tfrac{\alpha}{\min_{y \neq y_i} \hat{p}(y)},
$$
attained at the vertex placing mass $1 - \alpha$ on $y_i$ and $\alpha$ on $\argmin_{y \neq y_i} \hat{p}(y)$. This grows without bound as $\min_{y \neq y_i} \hat{p}(y) \to 0$.
We instead define $$L^\text{pess}(Q^\alpha_i, \hat{p}) = \sup_{p \in Q^\alpha_i} \mathbb{E}_{y \sim p}\big[-\log \hat{p}(y)\big].
$$
The objective is linear in $p$, so the supremum is attained at a vertex of the polytope $Q^\alpha_i$. The lower bound $p(y_i) \geq 1-\alpha$ is binding, and the residual mass $\alpha$ concentrates on the class with largest $-\log \hat{p}(y)$, yielding
$$L^\text{pess}(Q^\alpha_i, \hat{p}) = -(1-\alpha) \log \hat{p}(y_i) - \alpha \log \min_{y \neq y_i} \hat{p}(y).
$$
Consequently, $L^\text{pess}$ acts in an adversarial manner. While standard label smoothing distributes $\alpha$ mass uniformly across all incorrect classes, the pessimistic variant identifies the class the model is most confident is wrong and forces the model to hedge against that specific alternative.

\section{Proofs}\label{apx:proofs_recovery}

\subsection{Recovery of Existing Techniques}

We now prove that the unified framework~\eqref{eq:robust_framework} recovers each
of the techniques listed in Table~\ref{tab:special-cases} as a special case.

\begin{proposition}[Recovery of Label Smoothing]
\label{prop:label-smoothing}
Consider the robust framework~\eqref{eq:robust_framework}.
Let (i) the reference distribution be defined with smoothed labels in the $y$-space:
\[
P_{\mathrm{ref}}
=
\tilde{P}_{\mathrm{ls}}
=
\frac{1}{n} \sum_{i=1}^n \delta_{x_i} \otimes \tilde{y}_i,
\qquad
\tilde{y}_i
=
(1-\alpha)\, y_i
+
\alpha\, \frac{\mathbf{1}}{K},
\]
where $y_i$ are one-hot labels, $\alpha \in [0,1]$ is the smoothing parameter,
$K$ is the number of classes, and $\mathbf{1}/K$ denotes the uniform distribution
over labels; and (ii) the ambiguity set be a singleton:
\[
\U(P_{\mathrm{ref}}) = \{ P_{\mathrm{ref}} \}.
\]
Then the resulting optimization problem reduces to
\[
\hat{h}
=
\argmin_{h\in\Hspace}
\mathbb{E}_{(x,y)\sim \tilde{P}_{\mathrm{ls}}}
\bigl[ \ell(h(x), y) \bigr],
\]
which is exactly empirical risk minimization with label
smoothing~\citep{szegedyRethinkingInceptionArchitecture2016}.
\end{proposition}

\begin{proof}
Since $\U(P_{\mathrm{ref}}) = \{P_{\mathrm{ref}}\}$ is a singleton,
the aggregation over the ambiguity set is trivial regardless of the choice of
$\agg$:
\[
\agg_{Q \in \{P_{\mathrm{ref}}\}}
\mathbb{E}_{(x,y)\sim Q}[\ell(h(x),y)]
=
\mathbb{E}_{(x,y)\sim P_{\mathrm{ref}}}[\ell(h(x),y)].
\]
Substituting $P_{\mathrm{ref}} = \tilde{P}_{\mathrm{ls}} = \frac{1}{n}\sum_{i=1}^n
\delta_{x_i} \otimes \tilde{y}_i$, the expectation expands as
\[
\mathbb{E}_{(x,y)\sim \tilde{P}_{\mathrm{ls}}}[\ell(h(x),y)]
=
\frac{1}{n}\sum_{i=1}^n \mathbb{E}_{y \sim \tilde{y}_i}[\ell(h(x_i), y)]
=
\frac{1}{n}\sum_{i=1}^n \ell(h(x_i), \tilde{y}_i),
\]
where the last equality uses linearity of the loss in its second argument
(as is standard for cross-entropy with soft targets).
The minimizer of this objective is exactly $\hat{h}_{\mathrm{ERM}}$ trained
with smoothed labels $\tilde{y}_i$, which is by definition label smoothing.
\end{proof}

\begin{proposition}[Recovery of Label Relaxation]
\label{prop:label-relaxation}
Consider the robust framework~\eqref{eq:robust_framework} instantiated with
(i) the empirical distribution $P_{\mathrm{ref}} = \Pemp$,
(ii) an ambiguity set over labels defined via total variation with radius $\alpha$:
\[
\U(P_{\mathrm{ref}})
=
\Bigl\{
Q \in \mathcal{P}(\mathcal{X}\times\mathcal{Y})
:
\mathrm{TV}(Q_{\mathcal{Y}|x_i},\, \delta_{y_i}) \le \alpha
\;\;\forall\, i
\Bigr\},
\]
and (iii) the optimistic (infimum) aggregation.
Then the resulting optimization exactly recovers label
relaxation~\citep{DBLP:conf/aaai/LienenH21}:
\[
\hat{h}_{\mathrm{relax}}
=
\argmin_{h\in\Hspace}\;
\inf_{Q \in \U(P_{\mathrm{ref}})}
\mathbb{E}_{(x,y)\sim Q}[\ell(h(x),y)].
\]
\end{proposition}

\begin{proof}
Substituting into the framework and expanding the expectation over
$\Pemp = \frac{1}{n}\sum_{i=1}^n \delta_{(x_i,y_i)}$:
\[
\inf_{Q \in \U(P_{\mathrm{ref}})}
\mathbb{E}_{(x,y)\sim Q}[\ell(h(x),y)]
=
\inf_{Q \in \U(P_{\mathrm{ref}})}
\frac{1}{n}\sum_{i=1}^n
\mathbb{E}_{y \sim Q_{\mathcal{Y}|x_i}}[\ell(h(x_i), y)].
\]
Since the constraint on $Q$ decouples across data points---each
$Q_{\mathcal{Y}|x_i}$ is constrained independently via
$\mathrm{TV}(Q_{\mathcal{Y}|x_i}, \delta_{y_i}) \le \alpha$---the
infimum factors as
\[
=
\frac{1}{n}\sum_{i=1}^n
\inf_{\substack{q_i \in \mathcal{P}(\mathcal{Y}) \\
\mathrm{TV}(q_i,\, \delta_{y_i}) \le \alpha}}
\mathbb{E}_{y \sim q_i}[\ell(h(x_i), y)].
\]
The total variation constraint $\mathrm{TV}(q_i, \delta_{y_i}) \le \alpha$
is equivalent to $\sum_{y \neq y_i} q_i(y) \le \alpha$, which is precisely
the definition of the label ambiguity set
$\mathcal{Q}_i^\alpha = \{p \in \mathcal{P}(\mathcal{Y}) : \sum_{y \neq y_i}
p(y) \le \alpha\}$ of \citet{DBLP:conf/aaai/LienenH21}.
The objective therefore becomes
\[
\frac{1}{n}\sum_{i=1}^n
\inf_{q_i \in \mathcal{Q}_i^\alpha}
\mathbb{E}_{y \sim q_i}[\ell(h(x_i), y)],
\]
which is exactly the label relaxation objective.
\end{proof}

\begin{proposition}[Recovery of DRO]
\label{prop:dro}
Consider the robust framework~\eqref{eq:robust_framework} instantiated with
(i) the empirical distribution $P_{\mathrm{ref}} = \Pemp$,
(ii) an ambiguity set defined via a statistical distance $\Delta$
(e.g.\ Wasserstein) with radius $\varrho$:
\[
\U(P_{\mathrm{ref}})
=
\{Q \in \mathcal{P}(\mathcal{X}\times\mathcal{Y})
: \Delta(Q,\, \Pemp) \le \varrho\},
\]
and (iii) the pessimistic (supremum) aggregation.
Then the resulting optimization exactly recovers distributionally robust
optimization~\citep{blanchetDistributionallyRobustOptimization2025}:
\[
\hat{h}_{\mathrm{DRO}}
=
\argmin_{h\in\Hspace}
\sup_{Q \in \U(\Pemp)}
\mathbb{E}_{(x,y)\sim Q}[\ell(h(x),y)].
\]
\end{proposition}

\begin{proof}
Direct substitution of $P_{\mathrm{ref}} = \Pemp$, $\U(P_{\mathrm{ref}})$
as defined above, and $\agg = \sup$ into the unified
framework~\eqref{eq:robust_framework} immediately yields the DRO objective.
\end{proof}

\begin{proposition}[Recovery of DFO]
\label{prop:dfo}
Consider the robust framework~\eqref{eq:robust_framework} instantiated with
(i) the empirical distribution $P_{\mathrm{ref}} = \Pemp$,
(ii) an ambiguity set defined via a statistical distance $\Delta$ with
radius $\varrho$ as in Proposition~\ref{prop:dro},
and (iii) the optimistic (infimum) aggregation.
Then the resulting optimization exactly recovers distributionally favorable
optimization~\citep{jiangDistributionallyFavorableOptimization2024a}:
\[
\hat{h}_{\mathrm{DFO}}
=
\argmin_{h\in\Hspace}
\inf_{Q \in \U(\Pemp)}
\mathbb{E}_{(x,y)\sim Q}[\ell(h(x),y)].
\]
\end{proposition}

\begin{proof}
Direct substitution of $P_{\mathrm{ref}} = \Pemp$, $\U(P_{\mathrm{ref}})$
as in Proposition~\ref{prop:dro}, and $\agg = \inf$ into the unified
framework~\eqref{eq:robust_framework} immediately yields the DFO objective.
\end{proof}

\begin{proposition}[Recovery of Mixup]
\label{prop:mixup}
Consider the robust framework~\eqref{eq:robust_framework} instantiated with
(i) the reference distribution set to the Mixup vicinal distribution,
\[
P_{\mathrm{ref}}
=
\tilde{P}_{\mathrm{mix}}
=
\frac{1}{n^2}\sum_{i,j=1}^n \nu_\lambda(\,\cdot \mid (x_i,y_i),(x_j,y_j)),
\]
where $\nu_\lambda$ generates Mixup interpolations
$(\tilde{x},\tilde{y}) = \lambda(x_i,y_i) + (1-\lambda)(x_j,y_j)$
with $\lambda \sim \mathrm{Beta}(\alpha,\alpha)$,
and (ii) a singleton ambiguity set $\U(P_{\mathrm{ref}}) = \{P_{\mathrm{ref}}\}$.
Then the resulting optimization exactly recovers training with
Mixup~\citep{zhangMixupEmpiricalRisk2018}:
\[
\hat{h}_{\mathrm{Mixup}}
=
\argmin_{h \in \Hspace}
\mathbb{E}_{(x,y)\sim \tilde{P}_{\mathrm{mix}}}[\ell(h(x),y)].
\]
\end{proposition}

\begin{proof}
Since $\U(P_{\mathrm{ref}}) = \{P_{\mathrm{ref}}\}$ is a singleton,
the aggregation trivializes as in Proposition~\ref{prop:label-smoothing}:
\[
\agg_{Q \in \{\tilde{P}_{\mathrm{mix}}\}}
\mathbb{E}_{(x,y)\sim Q}[\ell(h(x),y)]
=
\mathbb{E}_{(x,y)\sim \tilde{P}_{\mathrm{mix}}}[\ell(h(x),y)].
\]
Expanding the expectation over $\tilde{P}_{\mathrm{mix}}$:
\[
\mathbb{E}_{(x,y)\sim \tilde{P}_{\mathrm{mix}}}[\ell(h(x),y)]
=
\frac{1}{n^2}\sum_{i,j=1}^n
\mathbb{E}_{\lambda \sim \mathrm{Beta}(\alpha,\alpha)}
\Bigl[
\ell\!\bigl(h(\lambda x_i + (1{-}\lambda)x_j),\;
\lambda y_i + (1{-}\lambda)y_j\bigr)
\Bigr],
\]
which is exactly the Mixup empirical risk of \citet{zhangMixupEmpiricalRisk2018}.
\end{proof}

\begin{proposition}[Recovery of VRM]
\label{prop:vrm}
Consider the robust framework~\eqref{eq:robust_framework} instantiated with
(i) the reference distribution set to the vicinal distribution,
\[
P_{\mathrm{ref}}
=
\tilde{P}_{\mathrm{vic}}
=
\frac{1}{n}\sum_{i=1}^n \nu(\,\cdot \mid x_i, y_i),
\]
and (ii) a singleton ambiguity set
$\U(P_{\mathrm{ref}}) = \{P_{\mathrm{ref}}\}$.
Then the resulting optimization exactly recovers vicinal risk
minimization~\citep{chapelleVicinalRiskMinimization2000}:
\[
\hat{h}_{\mathrm{VRM}}
=
\argmin_{h \in \Hspace}
\mathbb{E}_{(x,y)\sim \tilde{P}_{\mathrm{vic}}}[\ell(h(x),y)].
\]
\end{proposition}

\begin{proof}
Since $\U(P_{\mathrm{ref}}) = \{P_{\mathrm{ref}}\}$ is a singleton,
the aggregation trivializes as in Proposition~\ref{prop:label-smoothing}:
\[
\agg_{Q \in \{\tilde{P}_{\mathrm{vic}}\}}
\mathbb{E}_{(x,y)\sim Q}[\ell(h(x),y)]
=
\mathbb{E}_{(x,y)\sim \tilde{P}_{\mathrm{vic}}}[\ell(h(x),y)],
\]
which is exactly the VRM objective of \citet{chapelleVicinalRiskMinimization2000}.
\end{proof}

\subsection{Special cases}

We now show that for linear models many seemingly disparate robust learning techniques reduce to empirical
risk minimization with an additional regularization term.

\subsubsection{Regression}
We consider a supervised regression problem with data
\(
\{(x_i,y_i)\}_{i=1}^n\) with \(x_i \in \mathbb{R}^d, \; y_i \in \mathbb{R}.
\)
We study linear models of the form
\(
f_w(x) = w^\top x,
\)
where $w \in \mathbb{R}^d$ denotes the parameter vector. The learning objective is based on the squared loss
\[
\ell(f_w(x),y) = (y - w^\top x)^2 .
\] The corresponding ERM objective is
\[
R_{\mathrm{ERM}}(w)
=
\frac{1}{n}
\sum_{i=1}^n
(y_i - w^\top x_i)^2 .
\]

\begin{proposition}\label{prop:special_case_vrm_regression}{Shown by \citet{chapelleVicinalRiskMinimization2000}.}
Assume the vicinal distribution is defined by
$
\tilde{x} = x + \epsilon$, $\epsilon \sim \mathcal{N}(0,\sigma^2 I).$
Then the VRM objective is equivalent to
\[
R_{\mathrm{VRM}}(w)
=
R_{\mathrm{ERM}}(w)
+
\sigma^2 \|w\|^2 .
\]
\end{proposition}

\begin{proof}
The VRM objective minimizes the expected loss under the vicinal distribution:
\[
R_{\mathrm{VRM}}(w)
=
\frac{1}{n}\sum_{i=1}^n
\mathbb{E}_{\epsilon}
\big[
(y_i - w^\top (x_i + \epsilon))^2
\big].
\]

Expanding the square yields
\[
(y_i - w^\top x_i - w^\top \epsilon)^2
=
(y_i - w^\top x_i)^2
-2(y_i - w^\top x_i)w^\top \epsilon
+
(w^\top \epsilon)^2.
\]

Taking expectations with respect to $\epsilon$,
the cross term vanishes since $\mathbb{E}[\epsilon]=0$.
Thus,

\[
\mathbb{E}
\big[
(y_i - w^\top (x_i + \epsilon))^2
\big]
=
(y_i - w^\top x_i)^2
+
\mathbb{E}[(w^\top \epsilon)^2].
\]

Since $\epsilon \sim \mathcal{N}(0,\sigma^2 I)$,

\[
\mathbb{E}[(w^\top \epsilon)^2]
=
w^\top \mathbb{E}[\epsilon \epsilon^\top] w
=
\sigma^2 w^\top w
=
\sigma^2 \|w\|^2.
\]

Substituting into the VRM objective yields

\[
R_{\mathrm{VRM}}(w)
=
\frac{1}{n}\sum_{i=1}^n (y_i - w^\top x_i)^2
+
\sigma^2 \|w\|^2,
\]

which proves the claim.
\end{proof}

\begin{proposition}[Aggregation over a partition]
\label{prop:group-dro}
Let $\mathcal{G}$ partition the training set, let $\ell_g$ denote the mean loss
on cell $g$, and let $p_g>0$ with $\sum_g p_g=1$ be reference weights.
For $\tau>0$, define $A_\tau(\ell)
=
\tau\log\sum_{g\in\mathcal G} p_g
\exp\!\left(\frac{\ell_g}{\tau}\right).$
Then:
\begin{enumerate}
    \item as $\tau\to\infty$, $    A_\tau(\ell)\to \sum_g p_g\ell_g,
    $
    which equals ERM when $p_g$ are the empirical group proportions;

    \item as $\tau\to0$,
    $
    A_\tau(\ell)\to \max_g \ell_g,
    $
    recovering the Group DRO objective;

    \item for the finest partition and uniform $p_g=1/n$,
    $A_\tau$ is tilted ERM with tilt parameter $t=1/\tau$;

    \item adding offsets $a_g$ is equivalent, up to an additive constant, to
    replacing $p_g$ by
    $
    \widetilde p_g
    =
    \frac{p_g e^{a_g/\tau}}
         {\sum_h p_h e^{a_h/\tau}}.
    $
    Thus offsets modify the effective reference weights but not the
    pessimistic aggregation principle.
\end{enumerate}
\end{proposition}
\begin{proof}
Write
$
Z_\tau
=
\sum_{g\in\mathcal G} p_g e^{\ell_g/\tau}.
$

For (i), using
$
e^{\ell_g/\tau}
=
1+\frac{\ell_g}{\tau}+O(\tau^{-2}),
$
we obtain
$
Z_\tau
=
1+\frac{1}{\tau}\sum_g p_g\ell_g+O(\tau^{-2}).
$
Since $\log(1+x)=x+O(x^2)$,
$
A_\tau(\ell)
=
\tau\log Z_\tau
=
\sum_g p_g\ell_g+O(\tau^{-1}),
$
and therefore
$
\lim_{\tau\to\infty}A_\tau(\ell)
=
\sum_g p_g\ell_g.
$
If $p_g=n_g/n$ and $\ell_g$ is the mean loss in group $g$, then
$
\sum_g p_g\ell_g
=
\frac1n\sum_{i=1}^n \ell_i,
$
which is the ERM objective.

For (ii), let $m=\max_g\ell_g$. Then
$
A_\tau(\ell)
=
m+
\tau\log
\sum_g p_g
\exp\!\left(\frac{\ell_g-m}{\tau}\right).
$
Because $\ell_g-m\le0$, the term inside the logarithm is at most $1$.
For any maximizer $g^\star$ with $\ell_{g^\star}=m$, it is at least
$p_{g^\star}$. Hence
$
m+\tau\log p_{g^\star}
\le
A_\tau(\ell)
\le
m.
$
Taking $\tau\to0$ yields
$
A_\tau(\ell)\to m=\max_g\ell_g,
$
which is the Group DRO objective.

For (iii), let the partition consist of individual samples and let
$p_i=1/n$. Then
$
A_\tau(\ell)
=
\tau\log\left(
\frac1n\sum_{i=1}^n e^{\ell_i/\tau}
\right).
$
Setting $t=1/\tau$ gives
$
A_\tau(\ell)
=
\frac1t
\log\left(
\frac1n\sum_{i=1}^n e^{t\ell_i}
\right),
$
which is the tilted ERM objective.

Finally, for offsets $a_g$,
\[
\begin{aligned}
A_\tau(\ell+a)
&=
\tau\log
\sum_g p_g e^{(\ell_g+a_g)/\tau} \\
&=
\tau\log
\left[
\left(\sum_h p_h e^{a_h/\tau}\right)
\sum_g
\frac{p_g e^{a_g/\tau}}
     {\sum_h p_h e^{a_h/\tau}}
e^{\ell_g/\tau}
\right] \\
&=
\tau\log\sum_h p_h e^{a_h/\tau}
+
\tau\log\sum_g \widetilde p_g e^{\ell_g/\tau},
\end{aligned}
\]
where
$
\widetilde p_g
=
\frac{p_g e^{a_g/\tau}}
     {\sum_h p_h e^{a_h/\tau}}.
$
Thus the offsets alter only the effective reference weights, up to an additive
constant independent of the model parameters. Equivalently, by the Gibbs variational principle,
$
A_\tau(\ell+a)
=
\sup_{q\in\Delta_{\mathcal G}}
\left\{
\sum_g q_g(\ell_g+a_g)
-
\tau\,\mathrm{KL}(q\|p)
\right\},
$
whose optimizer is
$
q_g^\star
\propto
p_g\exp\!\left(\frac{\ell_g+a_g}{\tau}\right).
$
Hence the aggregation remains pessimistic; the offsets modify the effective
reference weighting.
\end{proof}
\begin{proposition}\label{prop:special_case_mixup_regression}
Let the Mixup samples be defined as
$\tilde{x} = \lambda x_i + (1-\lambda)x_j$, $\tilde{y} = \lambda y_i + (1-\lambda)y_j$,
where $(i,j)$ are drawn independently and uniformly from $[n]$ and
$\lambda \sim \mathrm{Beta}(\alpha,\alpha)$.
Let $c_\alpha = \mathbb{E}[\lambda(1-\lambda)]$. Then the Mixup objective satisfies
\[
R_{\mathrm{Mixup}}(w)
=
(1 - 2c_\alpha)\, R_{\mathrm{ERM}}(w)
+
2 c_\alpha\, \big(\bar{y} - w^\top \bar{x}\big)^2,
\]
where $\bar{x} = \tfrac{1}{n}\sum_i x_i$ and $\bar{y} = \tfrac{1}{n}\sum_i y_i$.
In particular, whenever the hypothesis class contains an intercept,
$R_{\mathrm{Mixup}}$ and $R_{\mathrm{ERM}}$ share the same minimizers.
\end{proposition}

\begin{proof}
Let $r_i = y_i - w^\top x_i$. Then
\[
\tilde{y} - w^\top \tilde{x} \;=\; \lambda r_i + (1-\lambda) r_j,
\]
and using the identity
$(\lambda a + (1-\lambda) b)^2 = \lambda a^2 + (1-\lambda) b^2 - \lambda(1-\lambda)(a-b)^2$,
\[
(\tilde{y} - w^\top \tilde{x})^2
\;=\;
\lambda r_i^2 + (1-\lambda) r_j^2 \;-\; \lambda(1-\lambda)(r_i - r_j)^2.
\]
Since $\mathbb{E}[\lambda] = \mathbb{E}[1-\lambda] = 1/2$ and
$\mathbb{E}[\lambda(1-\lambda)] = c_\alpha$, taking expectation over
$\lambda$ and $(i,j)$ gives
\[
R_{\mathrm{Mixup}}(w)
\;=\;
R_{\mathrm{ERM}}(w) \;-\; c_\alpha\, \mathbb{E}_{i,j}\!\left[(r_i - r_j)^2\right].
\]
Because $i$ and $j$ are independent and identically distributed,
\[
\mathbb{E}_{i,j}\!\left[(r_i - r_j)^2\right]
\;=\;
2\,\mathbb{E}_i[r_i^2] - 2\,(\mathbb{E}_i[r_i])^2
\;=\;
2\, R_{\mathrm{ERM}}(w) \;-\; 2\,(\bar{y} - w^\top \bar{x})^2,
\]
where the second equality uses $\bar r = \bar y - w^\top \bar x$. Substituting,
\[
R_{\mathrm{Mixup}}(w)
\;=\;
(1 - 2c_\alpha)\, R_{\mathrm{ERM}}(w) \;+\; 2c_\alpha\, (\bar{y} - w^\top \bar{x})^2.
\]
For the second claim, with intercept $b$ the second term becomes
$(\bar{y} - w^\top \bar{x} - b)^2 \ge 0$, vanishing iff $\bar r = 0$, which is
the first-order condition of $R_{\mathrm{ERM}}$ in $b$. Both terms are therefore
minimized simultaneously by any minimizer of $R_{\mathrm{ERM}}$. Since
$c_\alpha < 1/2$ for any $\mathrm{Beta}(\alpha,\alpha)$, the coefficient of
$R_{\mathrm{ERM}}$ is strictly positive, so under standard non-degeneracy of the
augmented feature matrix the minimizer is unique and coincides with that of
$R_{\mathrm{ERM}}$.
\end{proof}

\begin{remark}\label{rem:mixup_generative}
Under the generative model $y_i = w_\star^\top x_i + \varepsilon_i$ with
$\mathbb{E}[\varepsilon_i] = 0$, $\mathrm{Var}(\varepsilon_i) = \sigma^2$, and
$\varepsilon \perp x$, taking expectation over the noise yields the equivalent form
\[
\mathbb{E}_\varepsilon\!\left[R_{\mathrm{Mixup}}(w)\right]
\;=\;
\mathbb{E}_\varepsilon\!\left[R_{\mathrm{ERM}}(w)\right]
\;-\; c_\alpha\, (w - w_\star)^\top \Sigma_x (w - w_\star)
\;+\; \mathrm{const},
\]
with $\Sigma_x = \tfrac{1}{n^2}\sum_{i,j}(x_i - x_j)(x_i - x_j)^\top$. The
negative sign means Mixup contributes a concave bonus centered at $w_\star$
rather than a convex penalty, and the expected Mixup minimizer remains $w_\star$.
This recovers the observation of \citet{CarratinoCJV22} that Mixup has no
asymptotic effect on the linear least-squares minimizer.
\end{remark}

\begin{proposition}\label{prop:special_case_dro_regression}
Let $\varrho$ be the radius of the Wasserstein-$1$ ambiguity set with transport
cost $\|x - x'\|_q$, and let $q^*$ denote the conjugate exponent. For sufficiently
small $\varrho$,
\[
R_{\mathrm{DRO}}(w)
\;=\;
R_{\mathrm{ERM}}(w)
\;+\;
2\varrho\, \|w\|_{q^*}\, \cdot\, \frac{1}{n}\sum_{i=1}^n |y_i - w^\top x_i|
\;+\;
o(\varrho).
\]
The leading-order regularizer is the product of $\|w\|_{q^*}$ and the empirical
mean absolute residual; both factors depend on $w$.
\end{proposition}

\begin{proof}
By Theorem 1 of \citet{gao}, for sufficiently small $\varrho$,
\[
R_{\mathrm{DRO}}(w)
\;=\;
\frac{1}{n}\sum_{i=1}^n \ell(h_w(x_i), y_i)
\;+\;
\varrho \cdot \frac{1}{n}\sum_{i=1}^n \|\nabla_x \ell(h_w(x_i), y_i)\|_{q^*}
\;+\;
o(\varrho).
\]
For the squared loss,
$\nabla_x \ell(h_w(x), y) = -2(y - w^\top x) w$, so
\[
\|\nabla_x \ell(h_w(x), y)\|_{q^*}
\;=\;
2\, |y - w^\top x|\, \|w\|_{q^*}.
\]
Since $\|w\|_{q^*}$ does not depend on the index $i$, it factors out of the sum:
\[
\varrho \cdot \frac{1}{n}\sum_{i=1}^n \|\nabla_x \ell(h_w(x_i), y_i)\|_{q^*}
\;=\;
2\varrho\, \|w\|_{q^*}\, \cdot\, \frac{1}{n}\sum_{i=1}^n |y_i - w^\top x_i|.
\]
The mean absolute residual $\tfrac{1}{n}\sum_i |y_i - w^\top x_i|$ is a
$w$-dependent quantity and cannot be absorbed into a $w$-independent constant.
\end{proof}

\subsection{Classification}

We consider a supervised binary classification problem with data
$\{(x_i, y_i)\}_{i=1}^n$, $x_i \in \mathbb{R}^d$, $y_i \in \{-1, +1\}$.
We study linear classifiers $f_w(x) = w^\top x$ trained with the
squared loss $\ell(f_w(x), y) = (y - w^\top x)^2$, which corresponds
to least-squares classification. Since $y_i \in \{-1, +1\}$ implies
$y_i^2 = 1$, the squared loss coincides with the squared-margin loss
$(1 - y_i w^\top x_i)^2$. The corresponding ERM objective is
\[
R_{\mathrm{ERM}}(w) = \frac{1}{n}\sum_{i=1}^n (y_i - w^\top x_i)^2.
\]
We assume linear models include an intercept, incorporated by
augmenting features with a constant.

\begin{proposition}[VRM]\label{prop:special_case_vrm_classification}
Under the vicinal distribution
$\tilde{x} = x + \epsilon$, $\epsilon \sim \mathcal{N}(0, \sigma^2 I)$,
\[
R_{\mathrm{VRM}}(w) = R_{\mathrm{ERM}}(w) + \sigma^2 \|w\|^2.
\]
\end{proposition}

\begin{proof}
The squared loss takes the same form as in the regression setting
(Appendix~D.1) with the constraint $y_i \in \{-1,+1\}$, which is not
used in the proof of
Proposition~\ref{prop:special_case_vrm_regression}. The same argument
applies verbatim.
\end{proof}

\begin{proposition}[Mixup]\label{prop:special_case_mixup_classification}
Let the Mixup samples be defined as $\tilde{x} = \lambda x_i + (1-\lambda) x_j$,
$\tilde{y} = \lambda y_i + (1-\lambda) y_j$ with $(i,j)$ drawn
independently and uniformly from $[n]$ and
$\lambda \sim \mathrm{Beta}(\alpha, \alpha)$. Let
$c_\alpha = \mathbb{E}[\lambda(1-\lambda)]$. Then
\[
R_{\mathrm{Mixup}}(w)
= (1 - 2c_\alpha)\, R_{\mathrm{ERM}}(w)
+ 2c_\alpha\, (\bar{y} - w^\top \bar{x})^2,
\]
or equivalently $R_{\mathrm{Mixup}}(w) - R_{\mathrm{ERM}}(w)
= -2c_\alpha\, \mathrm{Var}_n(y - w^\top x)$. With intercept,
$R_{\mathrm{Mixup}}$ and $R_{\mathrm{ERM}}$ share the same
minimizers.
\end{proposition}

\begin{proof}
Under the Mixup convex combination, $\tilde y \in [-1, +1]$ even when
$y_i, y_j \in \{-1, +1\}$. The squared-loss form
$(\tilde y - w^\top \tilde x)^2 = (\lambda r_i + (1-\lambda) r_j)^2$
remains valid because $\tilde y$ enters only through the residual
$r_i = y_i - w^\top x_i$, and the algebraic identity used in the proof
of Proposition~\ref{prop:special_case_mixup_regression} does not
depend on the range of $y_i$. The same argument applies verbatim.
\end{proof}

\begin{proposition}[W-DRO]\label{prop:special_case_dro_classification}
Let $\varrho$ be the radius of the Wasserstein-$1$ ambiguity set with
transport cost $\|x - x'\|_q$, and let $q^*$ denote the conjugate
exponent. For sufficiently small $\varrho$,
\[
R_{\mathrm{DRO}}(w) = R_{\mathrm{ERM}}(w)
+ 2\varrho\, \|w\|_{q^*}\cdot \frac{1}{n}\sum_{i=1}^n |y_i - w^\top x_i|
+ o(\varrho).
\]
\end{proposition}

\begin{proof}
The gradient $\nabla_x \ell(f_w(x), y) = -2(y - w^\top x) w$ has the
same form for $y \in \{-1,+1\}$ as for $y \in \mathbb{R}$, so the
proof of Proposition~\ref{prop:special_case_dro_regression} applies
verbatim.
\end{proof}

\begin{proposition}[Label Smoothing under squared loss]
\label{prop:special_case_ls_classification}
For binary classification with $y_i \in \{-1, +1\}$, let label smoothing
replace each label by $\tilde y_i = (1-\alpha)\, y_i$ for some
$\alpha \in [0, 1]$. Under the squared loss,
\[
R_{\mathrm{LS}}(w)
\;=\; (1-\alpha)^2\, R_{\mathrm{ERM}}\!\big(w/(1-\alpha)\big),
\]
and consequently the minimizer satisfies
$w_{\mathrm{LS}}^{\star} = (1-\alpha)\, w_{\mathrm{ERM}}^{\star}$.
\end{proposition}

\begin{proof}
Substituting $\tilde y_i = (1-\alpha) y_i$,
\[
R_{\mathrm{LS}}(w)
= \frac{1}{n}\sum_{i=1}^n \big((1-\alpha) y_i - w^\top x_i\big)^2.
\]
Factoring $(1-\alpha)$ from each term and writing $u = w/(1-\alpha)$,
\[
R_{\mathrm{LS}}(w)
= \frac{1}{n}\sum_{i=1}^n (1-\alpha)^2 \big(y_i - u^\top x_i\big)^2
= (1-\alpha)^2\, R_{\mathrm{ERM}}(u).
\]
The minimizer of $R_{\mathrm{LS}}$ in $w$ is therefore characterized by
$\nabla_u R_{\mathrm{ERM}}(u) = 0$ at $u = w/(1-\alpha)$, i.e.\
$w/(1-\alpha) = w_{\mathrm{ERM}}^{\star}$, giving
$w_{\mathrm{LS}}^{\star} = (1-\alpha)\, w_{\mathrm{ERM}}^{\star}$.
\end{proof}
\subsection{Proofs for Interactions of Robustness Techniques}
\label{app:interaction_proofs}

We use the notation and assumptions of Section~\ref{sec:linear_interactions}. In particular, \(w=(b,\beta)\), where \(b\) is an unpenalized intercept and \(\beta\) contains the feature coefficients, and Gaussian VRM perturbs only the non-constant features. We write
\(
m(b,\beta)
:=
\bar y-b-\beta^\top\bar x
\)
for the empirical mean residual, $\Sigma
:=
\frac1n
\sum_{i=1}^n
(x_i-\bar x)(x_i-\bar x)^\top$ for the centered feature covariance, and $c
:=
\frac1n
\sum_{i=1}^n
(x_i-\bar x)(y_i-\bar y)$
for the centered feature-response covariance.

\begin{proof}[Proof of Proposition~\ref{prop:mixup_vrm_reparam}]
Applying Gaussian VRM with strength \(\gamma=\sigma^2\) to a mixed sample adds the term \(\gamma\|\beta\|_2^2\) to its expected squared loss. By Proposition~\ref{prop:special_case_mixup_regression},
\[
R_{\mathrm{Mixup}+\mathrm{VRM}}(b,\beta)
=
(1-2c_\alpha)
R_{\mathrm{ERM}}(b,\beta)
+
2c_\alpha m(b,\beta)^2
+
\gamma\|\beta\|_2^2.
\]
Since $\frac{\partial R_{\mathrm{ERM}}}{\partial b}
=
-2m(b,\beta)$
and
$\frac{\partial m(b,\beta)^2}{\partial b}
=
-2m(b,\beta),$
we obtain
\begin{align*}
\frac{\partial R_{\mathrm{Mixup}+\mathrm{VRM}}}{\partial b}
&=
-2(1-2c_\alpha)m(b,\beta)
-
4c_\alpha m(b,\beta) =
-2m(b,\beta).
\end{align*}
Hence every minimizer satisfies $m(b,\beta)=0,$ or equivalently, $b
=
\bar y-\beta^\top\bar x.$ On this set, the additional Mixup term vanishes and the objective reduces to $(1-2c_\alpha)
R_{\mathrm{ERM}}(b,\beta)
+
\gamma\|\beta\|_2^2.$ Because \(1-2c_\alpha>0\), dividing by this factor does not change the minimizer. We therefore obtain $R_{\mathrm{ERM}}(b,\beta)
+
\frac{\gamma}{1-2c_\alpha}
\|\beta\|_2^2,$ which is precisely the Gaussian-VRM objective with strength
$\widetilde\gamma
=
\frac{\gamma}{1-2c_\alpha}.$
Thus,
\[
w^\star_{\mathrm{Mixup}+\mathrm{VRM}}
(c_\alpha,\gamma)
=
w^\star_{\mathrm{VRM}}
\left(
\frac{\gamma}{1-2c_\alpha}
\right).
\]
\end{proof}

\begin{proof}[Proof of Corollary~\ref{cor:mixup_ls}]
Let $a
=
1-\alpha_{\mathrm{LS}}.$
Label smoothing replaces the labels by
$\widetilde y_i
=
ay_i.$
The resulting dataset $\widetilde{\mathcal D}
=
\{(x_i,ay_i)\}_{i=1}^n$
again defines a linear least-squares problem with an intercept.

By Proposition~\ref{prop:special_case_mixup_regression}, Mixup and ERM have the same minimizer for any such least-squares problem. Applying this result to \(\widetilde{\mathcal D}\) therefore gives $w^\star_{\mathrm{Mixup}+\mathrm{LS}}
=
w^\star_{\mathrm{LS}}.$ Proposition~\ref{prop:special_case_ls_classification} further yields $w^\star_{\mathrm{LS}}
=
a\,w^\star_{\mathrm{ERM}}.$ Consequently,
\[
w^\star_{\mathrm{Mixup}+\mathrm{LS}}
=
w^\star_{\mathrm{LS}}
=
(1-\alpha_{\mathrm{LS}})
w^\star_{\mathrm{ERM}}.
\]
\end{proof}

\begin{proof}[Proof of Proposition~\ref{prop:ls_vrm}]
Let $a
=
1-\alpha_{\mathrm{LS}}
\in(0,1).$

\emph{Scaling identity.}
For fixed \(\beta\), minimizing over the unpenalized intercept gives $b
=
\bar y-\beta^\top\bar x.$ Substituting the optimal intercept into the Gaussian-VRM objective gives
\[
\min_\beta
\frac1n
\sum_{i=1}^n
\left[
(y_i-\bar y)
-
\beta^\top(x_i-\bar x)
\right]^2
+
\gamma\|\beta\|_2^2.
\]
Its first-order condition is $
(\Sigma+\gamma I)\beta_\gamma
=
c,$ and hence
$\beta_\gamma
=
(\Sigma+\gamma I)^{-1}c.$
For \(\gamma=0\), this expression is understood under the assumption that \(\Sigma\) is nonsingular.

Under label smoothing, $\widetilde y_i
=
ay_i,$
so
$\overline{\widetilde y}
=
a\bar y$
and
$
\widetilde c
=
ac.$
Therefore,
$
\beta_{\mathrm{LS}+\mathrm{VRM}}
=
(\Sigma+\gamma I)^{-1}(ac)
=
a\beta_\gamma.$
The corresponding intercept satisfies
\(
b_{\mathrm{LS}+\mathrm{VRM}}
=
a\bar y
-
(a\beta_\gamma)^\top\bar x
=
a
\left(
\bar y-\beta_\gamma^\top\bar x
\right)
=
ab_\gamma.
\) Hence,
\[
w^\star_{\mathrm{LS}+\mathrm{VRM}}
(\alpha_{\mathrm{LS}},\gamma)
=
a\,w^\star_{\mathrm{VRM}}(\gamma).
\]

\emph{The composition cannot generally be reproduced by VRM alone.}
Suppose that $
\Sigma v_i
=
\lambda_i v_i$
and write
$c_i
=
v_i^\top c.$
Along eigendirection \(v_i\),$
v_i^\top\beta_\gamma
=
\frac{c_i}{\lambda_i+\gamma}.$
Assume that there are two eigendirections \(v_1,v_2\) with
$\lambda_1\neq\lambda_2$, $c_1\neq0$,
$c_2\neq0.$

Suppose, for contradiction, that some
\(\widetilde\gamma\geq0\) satisfies
$
w^\star_{\mathrm{VRM}}(\widetilde\gamma)
=
a\,w^\star_{\mathrm{VRM}}(\gamma).$
Then in particular
$
\beta_{\widetilde\gamma}
=
a\beta_\gamma,
$
which along every active eigendirection requires
$
\frac{c_i}{\lambda_i+\widetilde\gamma}
=
\frac{a c_i}{\lambda_i+\gamma}.
$
Since \(c_i\neq0\),
$
\widetilde\gamma
=
\frac{\gamma}{a}
+
\lambda_i
\left(
\frac1a-1
\right).
$
For \(a\in(0,1)\), the right-hand side depends strictly on \(\lambda_i\). Since
\(\lambda_1\neq\lambda_2\), the two active eigendirections require different values of \(\widetilde\gamma\), a contradiction.

\emph{The composition cannot generally be reproduced by label smoothing alone.}
Assume now that \(\gamma>0\) and that \(\Sigma\) is nonsingular. Without VRM, the ERM feature coefficients are
$\beta_0
=
\Sigma^{-1}c,$
and label smoothing with strength
\(\widetilde\alpha_{\mathrm{LS}}\) produces
$
\beta_{\mathrm{LS}}(\widetilde\alpha_{\mathrm{LS}})
=
\widetilde a\,\beta_0,$ and $
\widetilde a
=
1-\widetilde\alpha_{\mathrm{LS}}.
$
Suppose that label smoothing alone could reproduce the composed solution. Then for some
\(\widetilde a\in(0,1]\),
\[
a\beta_\gamma
=
\widetilde a\,\beta_0.
\]
Along an active eigendirection \(v_i\), this requires
\(
\frac{a c_i}{\lambda_i+\gamma}
=
\frac{\widetilde a c_i}{\lambda_i},
\)
and therefore
$
\widetilde a
=
a\frac{\lambda_i}{\lambda_i+\gamma}.
$
For \(\gamma>0\), the function
$
\lambda
\mapsto
\frac{\lambda}{\lambda+\gamma}
$
is strictly increasing. Hence the two distinct eigenvalues
\(\lambda_1\neq\lambda_2\) require two different values of
\(\widetilde a\), again a contradiction.

Thus, under these non-degeneracy conditions and for
\(\alpha_{\mathrm{LS}}\in(0,1)\) and \(\gamma>0\), the composition cannot be reproduced by retuning either VRM or label smoothing alone.
\end{proof}

\begin{proof}[Proof of Proposition~\ref{prop:mixup_counterexamples}]
\emph{(i) Cross-entropy.}
Consider the one-dimensional dataset
\[
\mathcal D
=
\{
(-1,0),(-1,0),(-1,0),(-1,1),
(1,0),(1,1),(1,1),(1,1)
\}.
\]
The empirical positive-label frequency is \(1/4\) at \(x=-1\) and \(3/4\) at \(x=1\). Consider the logistic model
\(
p_{b,w}(x)
=
\sigma(b+wx)\),
$\sigma(z)
=
\frac{1}{1+e^{-z}},$
trained with binary cross-entropy.

The ERM optimum matches the empirical conditional label frequencies:
$
\sigma(b-w)
=
\frac14,$ $
\sigma(b+w)
=
\frac34.$
Thus,
$
(b^\star_{\mathrm{ERM}},w^\star_{\mathrm{ERM}})
=
(0,\log3).
$

Now let
$
\lambda
\sim
\operatorname{Beta}(1,1)
$
and sample the two observations of a Mixup pair independently. For a mixed observation with input \(u\) and target \(\widetilde y\), the cross-entropy gradient with respect to the slope is
$
u\bigl(p(u)-\widetilde y\bigr).
$

Pairs with equal inputs have zero expected contribution to this gradient at the ERM solution because the predicted probabilities at \(x=\pm1\) equal the corresponding empirical label frequencies.

For pairs with opposite inputs, the mixed input \(u\) is uniformly distributed on \([-1,1]\), and the conditional expected mixed target is
$
t(u)
=
\mathbb E[\widetilde y\mid u]
=
\frac12+\frac{u}{4}.
$
At the ERM solution,
$
p(u)
=
\sigma(u\log3).
$
On \(u\in(0,1)\), this function is strictly concave and satisfies
$
p(0)=\frac12,$ $
p(1)=\frac34.
$
It therefore lies strictly above the chord joining its endpoint values:
$
p(u)
>
\frac12+\frac{u}{4}
=
t(u),$
$u\in(0,1).
$
Using
$
p(-u)=1-p(u),$

$t(-u)=1-t(u),
$
we likewise have
$
p(u)-t(u)<0,$
$
u\in(-1,0).$
Consequently,
$
u\bigl(p(u)-t(u)\bigr)>0
$
for every
$
u\in(-1,1)\setminus\{0\}.
$
The expected Mixup gradient with respect to the slope at
$
(b,w)=(0,\log3)
$
is therefore strictly positive. Hence the ERM solution is not stationary for the Mixup objective and cannot be a Mixup minimizer. Thus,
$
w^\star_{\mathrm{Mixup}}
\neq
w^\star_{\mathrm{ERM}}.
$

\emph{(ii) Nonlinearity in the input.}
Consider
\[
\mathcal D
=
\{(-1,1),(0,0),(1,1)\}
\]
and the predictor
$
h_{b,w}(x)
=
b+wx^2
$
with squared loss. ERM has the unique minimizer
$
(b^\star_{\mathrm{ERM}},w^\star_{\mathrm{ERM}})
=
(0,1),
$
which fits all three observations exactly.

For Mixup, let
$
\lambda
\sim
\operatorname{Beta}(1,1)
$
and sample pairs independently from \(\mathcal D\). At
$
(b,w)=(0,1),
$
the gradient of the Mixup risk with respect to the intercept is
$
-2\,
\mathbb E
\left[
\widetilde y-\widetilde x^2
\right].
$

For pairs with equal inputs, which occur with probability \(1/3\), the residual is zero. Opposite pairs \(\{-1,1\}\) occur with probability \(2/9\), and their expected residual is
$
\mathbb E
\left[
1-(2\lambda-1)^2
\right]
=
\frac23.
$
Pairs containing \(0\) and either \(1\) or \(-1\) occur with total probability \(4/9\), and their expected residual is
$
\mathbb E[\lambda-\lambda^2]
=
\frac16.
$
Therefore,
$
\mathbb E
\left[
\widetilde y-\widetilde x^2
\right]
=
\frac29\cdot\frac23
+
\frac49\cdot\frac16
=
\frac29.
$
The gradient with respect to the intercept at the ERM solution is consequently
$
-\frac49
\neq
0.
$
Thus the ERM solution is not stationary for the Mixup objective and therefore cannot minimize it. Hence,
$
w^\star_{\mathrm{Mixup}}
\neq
w^\star_{\mathrm{ERM}}.
$
\end{proof}

\section{Experimental Protocol}
\label{app:protocol}
\label{app:experimental_details}

This appendix details the protocol summarized in Section~\ref{sec:experiments}: data (\ref{app:data}), models and training (\ref{app:training}), the implementation of the techniques (\ref{app:implementation}), their search spaces (\ref{app:search_space}) and stances (\ref{app:stances}), the selection rules (\ref{app:selection_regimes}), the choice of training length (\ref{app:training_length}), the metrics (\ref{app:metrics}) and the compute (\ref{app:compute}).

\subsection{Datasets}
\label{app:data}
\paragraph{CIFAR-10-C.} The CIFAR-10 training images are split once into 40{,}000 for training and 10{,}000 for validation. Label noise, where used, is symmetric and applied to training and validation labels alike; all results in the main text are at label noise~0. Evaluation uses the clean test set and the 19 corruption families of CIFAR-10-C~\citep{hendrycksBenchmarkingNeuralNetwork2019} at five severities. The families fall into four groups, which define the folds of out-of-sample OOD validation: noise (Gaussian, shot, impulse, speckle), blur (defocus, glass, motion, zoom, Gaussian), weather (snow, frost, fog, brightness, spatter) and digital (contrast, elastic, pixelate, JPEG, saturate). Every environment, i.e.\ every pair of family and severity, is split once into two fixed halves of its images, one used for selection and one for reporting (Appendix~\ref{app:selection_regimes}).

\paragraph{Waterbirds.} The standard splits of \citet{Sagawa2020Distributionally}. The training group counts ($3498/184/56/1057$) match the published ones; the smallest group makes up $1.2\,\%$ of the training set.

\paragraph{HH-RLHF.} Preference pairs from \texttt{helpful-base} (training $30{,}331$, validation $6{,}508$, in-distribution test $6{,}505$) and from \texttt{harmless-base} ($42{,}211$) for evaluation under shift~\citep{bai2022training}. Pairs whose two truncated conversations are identical are removed; such ties arise because conversations are truncated to 512 tokens, which can cut the only turn in which the two responses differ. Validation data are drawn from \texttt{helpful-base} only. Both test sets are split once into two fixed halves for in-sample OOD validation.
\begin{table}[t]
\centering
\caption{Datasets used in our experiments and their licenses or usage terms.}
\label{tab:data-licenses}
{\small
\begin{tabular}{ll}
\toprule
Dataset & License / usage terms \\
\midrule
CIFAR-10-C & Apache License 2.0 \\
Waterbirds & Non-commercial research and educational use \\
HH-RLHF & MIT License \\
\bottomrule
\end{tabular}}
\end{table}
\subsection{Models and Training}
\label{app:training}
\paragraph{CIFAR-10-C.} ResNet-18 trained from scratch with SGD (Nesterov momentum $0.9$, learning rate $0.1$, weight decay $5\cdot10^{-4}$, batch size 128), two warmup epochs and a warmup-stable-decay schedule over 120 epochs with decay over the final $20\,\%$. We report the model after 120 epochs. Decayed branches of the schedule at 30 and 60 epochs are also stored; they enter only the calibration of the strength ranges (Appendix~\ref{app:search_space}).

\paragraph{Waterbirds.} ResNet-50 initialized from ImageNet weights, SGD (Nesterov momentum $0.9$, constant learning rate $10^{-3}$, weight decay $10^{-4}$, batch size 128) for 300 epochs, with random-resized-crop (scale $0.7$--$1$) and horizontal-flip augmentation. We report the final model; checkpoints at 50, 100 and 200 epochs are stored for the comparison with early stopping in Table~\ref{tab:waterbirds}.

\paragraph{HH-RLHF.} Each conversation is represented by the $\ell_2$-normalized final-layer CLS embedding (1024 dimensions) of the reward model \texttt{OpenAssistant/reward-model-deberta-v3-large-v2}~\citep{he2021deberta}, which is kept frozen. A scorer $s:\R^{1024}\to\R$ (one hidden layer of 128 ReLU units, no dropout or weight decay) is shared between the chosen and the rejected response, and the two scores form the logits of a two-class softmax, so that cross-entropy is the negative Bradley--Terry log-likelihood~\citep{bradley_rank_1952}, $P(a\succ b)=\sigma(s(\mathbf e_a)-s(\mathbf e_b))$. Training uses SGD with Nesterov momentum, learning rate $0.05$, batch size 128 and a warmup-stable-decay schedule over 40 epochs; we report the final model, with decayed branches at 10 and 20 epochs stored for the same purposes as on CIFAR-10-C. Training label noise is $0$, $0.1$ or $0.25$; with two classes, noise swaps the preference of a fraction of the pairs. The embedding model was itself trained on HH-RLHF among other data, so that the shift from \texttt{helpful-base} to \texttt{harmless-base} affects the scorer's training distribution but not the representation.

\paragraph{Recipe.} Weight decay, schedule and data augmentation are one fixed recipe per dataset and never vary between runs, so that only the techniques of the framework differ between configurations.

\subsection{Implementation of the Techniques}
\label{app:implementation}
All techniques run in one training operator (Section~\ref{sec:algorithm}), in the order Mixup, vicinal noise, target construction, input perturbation, label loss, aggregation. Switching a technique off removes it from the computation, including its use of random numbers, so that a coalition without it trains exactly as one in which it was never implemented.
\begin{itemize}
\item \textbf{Scale.} Input-space strengths are expressed relative to $\mathrm{scale}$, the median distance between a training input and a random partner in the space the model sees ($91.84$ for normalized CIFAR-10 pixels). This makes the strengths comparable across benchmarks.
\item \textbf{VRM.} Gaussian noise with per-coordinate standard deviation $\sigma=\sigma_{\mathrm{rel}}\cdot\mathrm{scale}/\sqrt{D}$ for inputs of dimension $D$, so that the injected noise has norm $\approx\sigma_{\mathrm{rel}}\cdot\mathrm{scale}$.
\item \textbf{Mixup.} Pairs within the batch with $\lambda\sim\mathrm{Beta}(\alpha,\alpha)$; inputs and one-hot targets are mixed with the same $\lambda$.
\item \textbf{W-DRO.} $\ell_2$ PGD~\citep{madry2018towards} with radius $\varepsilon=\varepsilon_{\mathrm{rel}}\cdot\mathrm{scale}$: three steps of size $\varepsilon/3$ along the normalized gradient, starting from the clean input, with projection onto the ball after each step and batch-norm statistics frozen during the attack. The attack maximizes the plain cross-entropy against the (possibly mixed) target, independently of the label stance.
\item \textbf{Label smoothing and label relaxation.} Smoothing mass $\alpha$. Label relaxation uses the credal set $\{q: q\ge(1-\alpha)T\}$ for a target vector $T$, with the loss taken at the member closest in KL divergence to the prediction~\citep{DBLP:conf/aaai/LienenH21}. For one-hot $T$ this is the original construction; stating it in terms of $T$ makes it well defined for the soft targets of Mixup.
\item \textbf{KL-tilt and group tilt.} For a target strength $s$, the temperature $\tau$ is solved per batch by bisection so that the effective sample size of the tilted weights equals $1+s(\mathrm{ESS}_{\mathrm{ref}}-1)$; the stance sets the sign (Section~\ref{sec:algorithm}). On CIFAR-10-C and HH-RLHF the cells are the samples. On Waterbirds they are the four groups, with the empirical group proportions as reference; group re-weighting is $s=1$ with the uniform reference. For comparison, the exponentiated-gradient (EG) solver of \citet{Sagawa2020Distributionally} is run at step sizes $\eta\in\{10^{-3},10^{-2},10^{-1}\}$.
\end{itemize}

\subsection{Search Spaces}
\label{app:search_space}
\paragraph{Calibration.} Before any coalition was trained, each technique was trained alone over a deliberately wide grid of strengths at label noise $0$ and $0.25$ (seed~0, 124 runs on CIFAR-10-C), and the grid was extended where the rule below required it. The resulting ranges were then fixed (Table~\ref{tab:ranges}). Calibration uses validation data only.

\paragraph{Range rule (CIFAR-10-C).} Let $g=\mathrm{CE}(\mathrm{ERM})-\mathrm{CE}(\text{strength})$ be the gain in noisy validation cross-entropy at the best of the three stored checkpoints. A strength is kept if (i)~$|g|>\tau$ at some noise level, where $\tau=\max(0.01,2\,\mathrm{SD})$ and SD is the seed standard deviation of ERM at that level ($\tau=0.010$ at noise~0, $0.030$ at noise $0.25$); (ii)~the model has not collapsed (noisy validation error below $0.80$); (iii)~it is not catastrophic, $g>g^\star-3\max(|g^\star|,\tau)$ at both noise levels, where $g^\star$ is the best gain of the technique; and (iv)~it does not lie on a plateau of an already kept strength. The range spans the kept strengths. If the strongest end of the grid is kept, the grid is extended on that side; the weak end needs no extension, since its limit is the coalition without the technique. The two stances of KL-tilt are calibrated separately.

\paragraph{Range rule (HH-RLHF).} On HH-RLHF no technique improves validation cross-entropy over ERM at any strength or noise level (largest gain $0.0001$, below the seed floor), so a rule anchored on the best validation gain does not define a search space. It is therefore replaced by an in-distribution cost budget: a strength is kept if its effect exceeds two seed standard deviations, the model has not collapsed (validation error below $0.48$), and its cost in validation cross-entropy is at most $5\,\%$ of that of ERM at every noise level. The cost of Mixup saturates below the budget, so its strong end is capped at $\alpha=8$. The resulting ranges are $\sigma_{\mathrm{rel}}\in[0.353,6.12]$, $\varepsilon_{\mathrm{rel}}\in[0.001,0.0316]$, Mixup $\alpha\in[0.178,8]$, label smoothing $\alpha\in[0.045,0.231]$, label relaxation $\alpha\in[0.015,0.134]$, and tilt strength $s\in[0.95,0.975]$ (pessimistic) and $s=0.975$ (optimistic). Every range begins where the technique first has a measurable effect.

\begin{table}[ht]
\centering
\caption{Strength ranges on CIFAR-10-C, fixed from noisy validation data before any coalition was trained. Strengths are sampled log-uniformly within the range, the tilt strength uniformly.}
\label{tab:ranges}
\small
\begin{tabular}{llcc}
\toprule
technique & strength & calibration grid & range \\
\midrule
VRM & $\sigma_{\mathrm{rel}}$ & $0.01$--$1$ & $0.01$--$0.063$ \\
W-DRO & $\varepsilon_{\mathrm{rel}}$ & $0.003$--$0.3$, extended & $0.00048$--$0.0075$ \\
Mixup & $\alpha$ & $0.05$--$8$ & $0.05$--$1.05$ \\
label smoothing & $\alpha$ & $0.005$--$0.5$ & $0.005$--$0.079$ \\
label relaxation & $\alpha$ & $0.01$--$0.5$, extended & $0.00096$--$0.229$ \\
KL-tilt, pessimistic & $s$ & $0.1$--$0.9$, extended & $0.9$--$0.975$ \\
KL-tilt, optimistic & $s$ & $0.1$--$0.9$, extended & $0.95$--$0.975$ \\
\bottomrule
\end{tabular}
\end{table}

Every range lies at the weak end of its calibration grid: on validation data, strong settings of every technique cost more than they gain.

\paragraph{Sobol design.} All configurations are points of one scrambled Sobol sequence over ten fixed coordinates: one per strength, one for the stance of KL-tilt, and three reserved, so that adding a technique does not move points already trained. Row $i$ of the sequence is the same point in every coalition, and a coalition reads only the coordinates of its active techniques. Every coalition is therefore evaluated at the same number of points, which is the equal budget referred to in Section~\ref{sec:experiments}. The CIFAR-10-C lattices use rows 0--7 (a first round of 13 coalitions used rows 0--15); HH-RLHF uses rows 0--31 of the same sequence.

\subsection{Stances}
\label{app:stances}
The stance of a technique (Section~\ref{sec:framework}) is fixed whenever the benchmark determines when a perturbation arises: test-time input shift is guarded against pessimistically, and training label noise is treated optimistically, by label relaxation, with label smoothing as the neutral member. On the sample axis the benchmark does not decide, since a high-loss sample may be mislabeled or merely rare; there both stances enter the search and are chosen by the selection rule. Table~\ref{tab:stances} lists the resulting design. W-DFO and the pessimistic label loss of Appendix~\ref{app:robust_label_loss} are members of the framework that we do not evaluate; they would apply to contaminated training inputs and to label shift at deployment, respectively, and neither occurs in these benchmarks.

\begin{table}[ht]
\centering
\caption{Techniques and stances per benchmark and stage of the training procedure (Figure~\ref{fig:robust_pipeline}). ``tuned'': both stances enter the search.}
\label{tab:stances}
\small
\begin{tabular}{llll}
\toprule
stage & CIFAR-10-C & HH-RLHF & Waterbirds \\
\midrule
1 reference & VRM, Mixup & VRM, Mixup & -- \\
2 input & W-DRO (pessimistic) & W-DRO (pessimistic) & -- \\
3 label & LS (neutral) or LR (optimistic) & LS or LR & -- \\
4 sample & KL-tilt, tuned & KL-tilt, tuned & group tilt, $s\in[0,1]$ \\
\bottomrule
\end{tabular}
\end{table}

\paragraph{Coalitions.} The players of a lattice are the five techniques VRM, Mixup, W-DRO, KL-tilt and one label technique, which gives $2^5=32$ coalitions including ERM. On CIFAR-10-C the label technique is label smoothing; label relaxation enters the first round and the five-seed retraining (Table~\ref{tab:finals}). On HH-RLHF there are two lattices, with label smoothing and with label relaxation, which share their 16 label-free coalitions (48 coalitions in all).

\FloatBarrier
\subsection{Selection Rules}
\label{app:selection_regimes}
Every selection rule chooses one configuration, i.e.\ one (coalition, Sobol point) pair, from a candidate pool. For a single technique the pool is its own points, for \emph{best single, selected} it is ERM and the points of all single techniques, and for \textsc{Joint} it is the points of all coalitions.

\paragraph{CIFAR-10-C.} \emph{ID validation} selects the lowest value of the reported metric on the validation split. \emph{Out-of-sample OOD validation} holds out each of the four corruption groups in turn and selects the lowest value of the reported metric on the other three groups (mean over their families at severity~3). \emph{In-sample OOD validation} does the same on the held-out group itself and serves as an oracle. Both OOD rules select on the selection half of each environment, and every reported number, the oracle's included, is computed on the reporting half. The two halves are disjoint within an environment; across environments, which corrupt the same 10{,}000 base images, the selection half of one environment and the reporting half of another share about half of their base images. On the lattice, selecting on the full panel instead (complete overlap) leaves the share of the ID-to-oracle gap that out-of-sample OOD validation closes unchanged. All three rules report the held-out group's mean over its families at severity~3, averaged over the four folds; this weights the four groups equally and differs slightly from the mean over the 19 families used elsewhere.

\paragraph{HH-RLHF.} ID validation selects the lowest value per metric on the \texttt{helpful-base} validation split; in-sample OOD validation selects the lowest value of the reported metric on one half of \texttt{harmless-base} and reports on the other half. HH-RLHF provides only one shifted environment, so there is no out-of-sample OOD rule.

\paragraph{Waterbirds.} The validation split of \citet{Sagawa2020Distributionally} is drawn from the test distribution, so that selecting on it would be an in-sample OOD rule in the above sense. We therefore report each setting of the group tilt at its final model and, for comparison with \citet{Sagawa2020Distributionally}, at the checkpoint with the highest worst-group validation accuracy (Table~\ref{tab:waterbirds}).

\subsection{Training Length}
\label{app:training_length}
A pilot study on CIFAR-10 (ERM, label smoothing and Mixup at several strengths) motivates two protocol choices. First, a fixed long horizon inflates the gains of the techniques: at label noise $0.25$, the gain of Mixup ($\alpha=1$) over ERM in clean-label error is $1.8$, $5.9$ and $7.2$ percentage points at fixed horizons of 60, 120 and 180 epochs, but $1.9$ points when the length of each run is selected on validation data, since most of the apparent effect is ERM trained past its best point. Second, early stopping is biased against regularizers: at label noise $0.25$, the ranking of configurations after 20 epochs has a Spearman correlation of $-0.8$ with the ranking after 180 epochs, so that successive halving would discard the eventually best configuration at any budget below 60 of 120 epochs. We therefore train every configuration to the end of its schedule, without pruning, and report the final model.

\subsection{Evaluation Metrics}
\label{app:metrics}
Classification error; the Brier score~\citep{brier1950verification}, summed over classes; and the expected calibration error (ECE) with 15 equal-width confidence bins~\citep{gouCalibration2017}. On HH-RLHF a pair counts as an error if the scorer prefers the rejected response, and Brier score and ECE are computed from the two-class probabilities as on CIFAR-10-C. On Waterbirds we report worst-group and average test accuracy. On CIFAR-10-C a reported number is the mean over the 19 families at severity~3 unless stated otherwise.

\FloatBarrier
\subsection{Compute}
\label{app:compute}
End-to-end trainings ran on NVIDIA A100 GPUs, partly on 3g.40gb MIG slices. A CIFAR-10 ResNet-18 epoch takes about 7\,s for ERM and 25\,s with three-step PGD; a 300-epoch Waterbirds run takes about 42 minutes on one slice.

\FloatBarrier
\section{Additional Experimental Results}
\label{app:additional_e2e}

\subsection{CIFAR-10-C: Full Lattice and Retraining}
\label{app:cifar_lattice}
Table~\ref{tab:lattice_full} lists all 32 coalitions of the lattice, on which Figure~\ref{fig:cifar_lattice} is based. Two observations complement Section~\ref{sec:q2}. First, validation cross-entropy and shifted performance rank the coalitions differently: the coalition with the lowest validation cross-entropy (VRM, W-DRO and KL-tilt) is not the best on any shifted metric. Second, the five-technique coalition attains its value at its third Sobol point, so that eight and sixteen points give it the same value; the equal-budget design thus does not under-tune the large coalitions. Table~\ref{tab:finals} retrains six coalitions, ERM included, at five seeds. It measures the noise floor used throughout and confirms the size of the composition effect: the full coalition with label smoothing reduces shifted error from $0.246$ to $0.162$ and ECE from $0.158$ to $0.035$, with seed standard deviations of at most $0.004$.

\begin{table}[ht]
\centering
\caption{All 32 coalitions of the CIFAR-10-C lattice, each at its Sobol point with the lowest validation cross-entropy (val.\ CE). Shifted metrics: mean over the 19 families at severity~3, lower is better. Best per column \textbf{bold}.}
\label{tab:lattice_full}
\small
\begin{tabular}{lccccc}
\toprule
coalition & size & val.\ CE & error & Brier & ECE \\
\midrule
ERM & 0 & .201 & .248 & .397 & .158 \\
W-DRO & 1 & .188 & .208 & .333 & .128 \\
Mixup & 1 & .190 & .238 & .358 & .103 \\
VRM & 1 & .192 & .242 & .385 & .152 \\
KL-tilt & 1 & .179 & .247 & .382 & .137 \\
LS & 1 & .186 & .251 & .396 & .149 \\
W-DRO + KL-tilt & 2 & .187 & .173 & .259 & .079 \\
Mixup + W-DRO & 2 & .184 & .181 & .268 & .058 \\
VRM + W-DRO & 2 & .185 & .184 & .289 & .107 \\
W-DRO + LS & 2 & .183 & .189 & .291 & .091 \\
VRM + Mixup & 2 & .191 & .228 & .344 & .098 \\
VRM + LS & 2 & .185 & .239 & .374 & .137 \\
Mixup + KL-tilt & 2 & .191 & .241 & .350 & .080 \\
Mixup + LS & 2 & .186 & .242 & .363 & .103 \\
VRM + KL-tilt & 2 & .183 & .245 & .376 & .134 \\
LS + KL-tilt & 2 & .179 & .251 & .374 & .113 \\
Mixup + W-DRO + LS & 3 & .192 & \textbf{.157} & \textbf{.231} & .038 \\
VRM + Mixup + W-DRO & 3 & .189 & .161 & .237 & .043 \\
Mixup + W-DRO + KL-tilt & 3 & .196 & .163 & .236 & .036 \\
VRM + W-DRO + KL-tilt & 3 & \textbf{.176} & .167 & .252 & .076 \\
W-DRO + LS + KL-tilt & 3 & .188 & .171 & .255 & .069 \\
VRM + W-DRO + LS & 3 & .184 & .185 & .283 & .092 \\
VRM + Mixup + LS & 3 & .191 & .230 & .343 & .091 \\
VRM + Mixup + KL-tilt & 3 & .190 & .235 & .345 & .084 \\
VRM + LS + KL-tilt & 3 & .180 & .241 & .363 & .118 \\
Mixup + LS + KL-tilt & 3 & .198 & .243 & .357 & .088 \\
VRM + Mixup + W-DRO + LS & 4 & .191 & .158 & .231 & .037 \\
Mixup + W-DRO + LS + KL-tilt & 4 & .193 & .162 & .233 & .036 \\
VRM + Mixup + W-DRO + KL-tilt & 4 & .192 & .164 & .237 & \textbf{.034} \\
VRM + W-DRO + LS + KL-tilt & 4 & .185 & .172 & .257 & .071 \\
VRM + Mixup + LS + KL-tilt & 4 & .194 & .243 & .358 & .088 \\
VRM + Mixup + W-DRO + LS + KL-tilt & 5 & .197 & .163 & .235 & .034 \\
\bottomrule
\end{tabular}
\end{table}

\begin{table}[ht]
\centering
\caption{CIFAR-10-C, five seeds per coalition: mean $\pm$ std (severity~3, mean over the 19 families; lower is better), at configurations selected on ID validation in the first round. ``all five'': VRM, Mixup, W-DRO, KL-tilt and the stated label technique.}
\label{tab:finals}
\small
\begin{tabular}{lccc}
\toprule
coalition & error & Brier & ECE \\
\midrule
ERM & $.246_{\pm.003}$ & $.396_{\pm.007}$ & $.158_{\pm.004}$ \\
W-DRO & $.200_{\pm.003}$ & $.317_{\pm.006}$ & $.121_{\pm.004}$ \\
VRM + W-DRO & $.182_{\pm.004}$ & $.288_{\pm.007}$ & $.106_{\pm.004}$ \\
LS + VRM + W-DRO & $.183_{\pm.004}$ & $.284_{\pm.006}$ & $.094_{\pm.003}$ \\
all five, LR & $.178_{\pm.002}$ & $.274_{\pm.003}$ & $.075_{\pm.001}$ \\
all five, LS & $\mathbf{.162_{\pm.002}}$ & $\mathbf{.234_{\pm.003}}$ & $\mathbf{.035_{\pm.002}}$ \\
\bottomrule
\end{tabular}
\end{table}

\FloatBarrier
\subsection{CIFAR-10-C: Techniques by Corruption Family}
\label{app:per_corruption}
For each corruption family, the main effect of a technique is the mean error of the 16 coalitions of the lattice in Table~\ref{tab:lattice_full} that contain it minus the mean error of the 16 that do not, each coalition at its Sobol point with the lowest validation cross-entropy. Its standard error, $\mathrm{sd}_f\sqrt{2/16}$ with the per-family seed standard deviation $\mathrm{sd}_f$ of the five-seed retraining (Table~\ref{tab:finals}), is a lower bound, since interactions do not cancel. Each effect is tested with a two-sided normal test, and the family-wise error rate over all $5\times19=95$ effects is controlled at $5\,\%$ with Holm's procedure~\citep{holmSimpleSequentiallyRejective1979}.

Read per corruption family rather than averaged, two of the five techniques change sign (Figure~\ref{fig:mechanism_by_corruption}): W-DRO reduces error on 15 families and increases it on contrast ($+0.062$) and fog ($+0.026$), and KL-tilt reduces it on five and increases it on four, each significant after a Holm correction over the 95 technique--family effects; Mixup increases error on one family (saturate, $+0.002$). This structure is the reason for evaluating out-of-sample OOD validation with whole corruption groups held out. A per-family best technique, read off the evaluation data, is not evidence that a selection rule can exploit this structure; the held-out evaluation in Appendix~\ref{app:folds} is.

\begin{figure}[ht]
\centering
\includegraphics[width=0.9\linewidth]{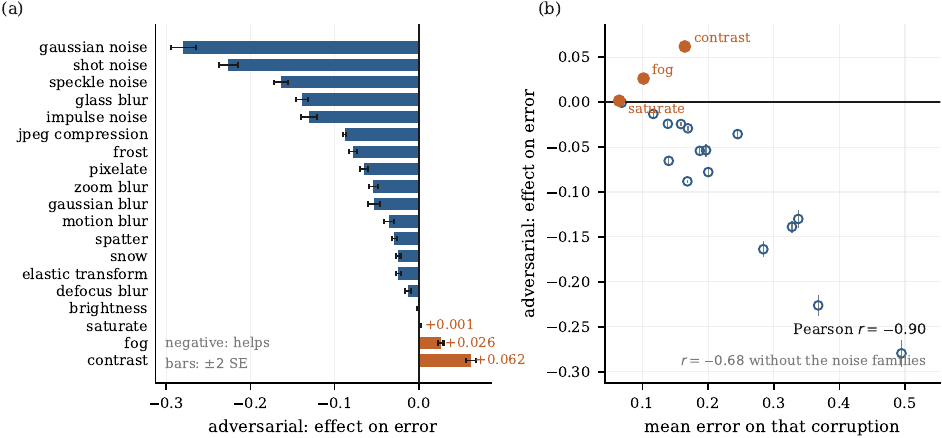}
\caption{Main effect of each technique on shifted error per corruption family (CIFAR-10-C, seed~0, final models, severity~3), relative to the per-family noise floor measured from five seeds.}
\label{fig:mechanism_by_corruption}
\end{figure}

\FloatBarrier
\subsection{CIFAR-10-C: Selection per Held-Out Corruption Group}
\label{app:folds}
Table~\ref{tab:q4_folds} gives the per-group numbers behind the \textsc{Joint} rows for the lattices of seeds 2 and 3, and adds a transfer across severities. Out-of-sample OOD validation improves on ID validation for every held-out group and on both seeds. The difference is largest on the noise group, where the configuration chosen by ID validation has an error of $0.46$--$0.50$, and smallest on the weather and digital groups, where the choice of ID validation is already reasonable. In ECE, out-of-sample OOD validation is within $0.007$ of the oracle on every group.

\begin{table}[ht]
\centering
\caption{\textsc{Joint} on CIFAR-10-C per held-out corruption group (severity~3, 3 lattices), selected on ID validation (ID), on the other three groups (out) or on the held-out group (in, oracle); lower is better. Last row per seed: severity~5, with ``out'' selecting on severity~1 of all families.}
\label{tab:q4_folds}
\small
\begin{tabular}{llcccccc}
\toprule
& & \multicolumn{3}{c}{error} & \multicolumn{3}{c}{ECE} \\
\cmidrule(lr){3-5}\cmidrule(lr){6-8}
seed & held out & ID & out & in & ID & out & in \\
\midrule
2 & noise & .489 & .214 & .138 & .266 & .027 & .023 \\
 & blur & .246 & .144 & .145 & .100 & .022 & .022 \\
 & weather & .152 & .113 & .107 & .049 & .027 & .028 \\
 & digital & .153 & .112 & .102 & .027 & .023 & .016 \\
 & sev.~1$\to$5 & .431 & .267 & .260 & .219 & .109 & .059 \\
\midrule
3 & noise & .358 & .195 & .141 & .284 & .026 & .023 \\
 & blur & .203 & .142 & .134 & .108 & .024 & .024 \\
 & weather & .127 & .128 & .106 & .049 & .027 & .027 \\
 & digital & .123 & .134 & .099 & .028 & .022 & .016 \\
 & sev.~1$\to$5 & .379 & .293 & .251 & .232 & .098 & .060 \\
\midrule
4 & noise & .493 & .237 & .140 & .233 & .024 & .024 \\
 & blur & .248 & .136 & .135 & .100 & .022 & .022 \\
 & weather & .148 & .115 & .111 & .047 & .025 & .025 \\
 & digital & .143 & .110 & .100 & .026 & .022 & .013 \\
 & sev.~1$\to$5 & .444 & .266 & .251 & .221 & .108 & .054 \\
\bottomrule
\end{tabular}
\end{table}

\subsection{Configurations Selected by \textsc{Joint}}
\label{app:picks}
Table~\ref{tab:picks} counts, over all selections of \textsc{Joint} (four folds, three metrics and two seeds per rule on CIFAR-10-C; three metrics on HH-RLHF), how often each technique is active and how many techniques are active. Under ID validation \textsc{Joint} selects KL-tilt alone on CIFAR-10-C and ERM on HH-RLHF. The OOD rules select compositions of two to five techniques; on CIFAR-10-C, W-DRO is part of all 48 of their selections and Mixup of 47.

\begin{table}[ht]
\centering
\caption{Configurations selected by \textsc{Joint}: in how many selections each technique is active, and how many techniques are active. OOD out/in: out-of-sample/in-sample OOD validation. KL-tilt stance: pessimistic in 32 of 35 (CIFAR-10-C), optimistic in 2 of 3 (HH-RLHF).}
\label{tab:picks}
{\small
\setlength{\tabcolsep}{3pt}
\begin{tabular}{lcccccc}
\toprule
rule & VRM & Mixup & W-DRO & KL & LS/LR & size \\
\midrule
\multicolumn{7}{l}{\emph{CIFAR-10-C}, 3 seeds (36 selections per rule)} \\
ID & 12 & 20 & 16 & 0 & 16 & 1--2 \\
OOD, out & 22 & 36 & 36 & 6 & 15 & 2--5 \\
OOD, in & 17 & 34 & 36 & 9 & 17 & 2--5 \\
\midrule
\multicolumn{7}{l}{\emph{HH-RLHF} (3 selections per rule)} \\
ID & 1 & 1 & 1 & 2 & 1 & 0--4 \\
OOD, in & 0 & 3 & 2 & 3 & 2 & 3--4 \\
\bottomrule
\end{tabular}}
\end{table}

\FloatBarrier
\subsection{Stance Chosen on the Sample Axis}
\label{app:stance_results}
On the sample axis both stances of KL-tilt enter the search (Appendix~\ref{app:stances}). In the lattice, selecting on validation cross-entropy within each of the 16 coalitions that contain KL-tilt chooses the pessimistic stance in all 16. On three seeds, the selected configurations that contain KL-tilt use the pessimistic stance in 12 of 15 cases; on HH-RLHF, three of the five selected configurations use the optimistic stance (Table~\ref{tab:picks}). The stance coordinate alternates between consecutive Sobol points, so both stances are available to every coalition at every budget.

\FloatBarrier
\subsection{Waterbirds}
\label{app:waterbirds}
Table~\ref{tab:waterbirds} reports every setting of the group tilt at the final model and, for comparison with the protocol of \citet{Sagawa2020Distributionally}, at the checkpoint with the highest worst-group validation accuracy. At $s=1$ the tilt reproduces ERM ($0.762$ against $0.766$ worst-group accuracy), and every group technique improves on ERM, although at a lower absolute level than reported by \citet{Sagawa2020Distributionally} ($0.82$ against $0.91$ with early stopping), since we use one fixed training recipe rather than their per-technique tuning. Intermediate strengths match the best exponentiated-gradient run ($0.824$ against $0.814$ for final models; $0.830$ against $0.823$ with early stopping). With three seeds, the difference between $s=0.25$ and the closed-form vertex $s=0$ ($+0.027$) is within seed variation (Welch test, $p=0.14$), whereas the difference to the reference $s=1$ ($+0.062$, $p=0.02$) is significant. The interior of the tilt thus matches Group DRO and its solver and improves on ERM; it does not significantly improve on the vertex. Early stopping raises worst-group accuracy by at most $0.014$ and leaves the ordering intact.

The closed-form tilt recovers the objective of Group DRO but not its minibatch estimate: at $s=0$, the exponentiated-gradient solver is better by about $0.02$. The smallest group contains 56 of 4{,}795 training images and is absent from about a fifth of all batches of 128. In those batches the closed form tilts toward the worst group present rather than the worst group overall, whereas the time-averaged update of \citet{Sagawa2020Distributionally} is not affected. When groups are rare, the closed form therefore calls for full-batch aggregation or a time-averaged estimate of the group losses.

\begin{table}[ht]
\centering
\caption{\textbf{Group shift} on Waterbirds, ResNet-50 trained end to end: test worst-group and average accuracy (higher is better), mean $\pm$ std over three seeds, for the final model and for the checkpoint with the best worst-group validation accuracy~\citep{Sagawa2020Distributionally}. Tilt strength $s$: $0$ is Group DRO (closed form), $1$ the reference (ERM). EG: exponentiated-gradient solver with step size~$\eta$. Best \textbf{bold}, second \underline{underlined}.}
\label{tab:waterbirds}
\small
\begin{tabular}{lcccc}
\toprule
& \multicolumn{2}{c}{final model} & \multicolumn{2}{c}{early-stopped} \\
\cmidrule(lr){2-3} \cmidrule(lr){4-5}
& worst $\uparrow$ & average $\uparrow$ & worst $\uparrow$ & average $\uparrow$ \\
\midrule
ERM & $.766_{\pm.013}$ & $.910_{\pm.004}$ & $.766_{\pm.013}$ & $.910_{\pm.004}$ \\
tilt $s{=}1$ (ref.) & $.762_{\pm.011}$ & $.910_{\pm.004}$ & $.763_{\pm.009}$ & $.909_{\pm.001}$ \\
tilt $s{=}0.75$ & $.808_{\pm.014}$ & $\underline{.933}_{\pm.003}$ & $.819_{\pm.004}$ & $.933_{\pm.002}$ \\
tilt $s{=}0.5$ & $\underline{.817}_{\pm.026}$ & $\mathbf{.936}_{\pm.005}$ & $\mathbf{.831}_{\pm.031}$ & $\mathbf{.936}_{\pm.008}$ \\
tilt $s{=}0.25$ & $\mathbf{.824}_{\pm.021}$ & $\underline{.933}_{\pm.004}$ & $\underline{.830}_{\pm.016}$ & $\underline{.934}_{\pm.004}$ \\
tilt $s{=}0$ (Group DRO) & $.797_{\pm.008}$ & $.931_{\pm.005}$ & $.801_{\pm.012}$ & $.931_{\pm.005}$ \\
re-weighting & $.806_{\pm.027}$ & $.929_{\pm.003}$ & $.813_{\pm.016}$ & $.930_{\pm.005}$ \\
EG, $\eta{=}0.001$ & $.814_{\pm.024}$ & $.928_{\pm.003}$ & $.822_{\pm.011}$ & $.929_{\pm.005}$ \\
EG, $\eta{=}0.01$ & $.810_{\pm.025}$ & $.926_{\pm.001}$ & $.823_{\pm.006}$ & $.928_{\pm.004}$ \\
EG, $\eta{=}0.1$ & $.806_{\pm.018}$ & $.924_{\pm.004}$ & $.809_{\pm.012}$ & $.924_{\pm.005}$ \\
\bottomrule
\end{tabular}
\end{table}

\FloatBarrier

\end{document}